\documentclass{article}
\usepackage{iclr2016_conference,times}

\usepackage[page,header]{appendix}
\usepackage{lipsum}
\usepackage{amsmath}
\usepackage{amssymb}
\usepackage{amsthm}
\usepackage{thmtools}
\usepackage{mathtools}
\usepackage{bm}
\usepackage{relsize}
\usepackage{natbib}
\usepackage{url}
\usepackage{algorithm}
\usepackage{algorithmic}
\usepackage{wrapfig}
\usepackage{subfigure}
\usepackage{booktabs}
\usepackage{multirow}
\usepackage{colortbl} 
\usepackage{bigdelim}

\usepackage{rotating}
\usepackage{color, colortbl}

\definecolor{mColor1}{rgb}{0.95,0.95,0.95}

\usepackage{hyperref}

\usepackage{array}

\newtheorem{theorem}{Theorem}

\newtheorem{lemma}{Lemma}

\newcommand\Ba{\bm{a}}
\newcommand\Bb{\bm{b}}
\newcommand\Bc{\bm{c}}

\newcommand\Bg{\bm{g}}

\newcommand\Bu{\bm{u}}

\newcommand\Bw{\bm{w}}
\newcommand\Bx{\bm{x}}

\newcommand\Bz{\bm{z}}
\newcommand\BA{\bm{A}}
\newcommand\BB{\bm{B}}
\newcommand\BC{\bm{C}}

\newcommand\BF{\bm{F}}

\newcommand\BK{\bm{K}}

\newcommand\BM{\bm{M}}

\newcommand\BS{\bm{S}}

\newcommand\BU{\bm{U}}

\newcommand\BX{\bm{X}}

\newcommand\BZe{\bm{0}}

\newcommand\bbR{\mathbb{R}}

\newcommand\EXP{\mathbf{\mathrm{E}}}

\newcommand\VAR{\mathbf{\mathrm{Var}}}
\newcommand\COV{\mathbf{\mathrm{Cov}}}

\iclrfinalcopy

\title{Fast and Accurate Deep Network Learning by
  Exponential Linear Units (ELUs)}

\author{Djork-Arn{\'e} Clevert,
Thomas Unterthiner
\&
Sepp Hochreiter\\
Institute of Bioinformatics\\
Johannes Kepler University, Linz, Austria\\
\texttt{\{okko,unterthiner,hochreit\}@bioinf.jku.at}
}

\begin{document}

\maketitle

\begin{abstract}
We introduce the ``exponential linear unit'' (ELU)
which speeds up learning in deep neural networks and
leads to higher classification accuracies.
Like rectified linear units (ReLUs), leaky ReLUs
(LReLUs) and parametrized ReLUs (PReLUs),
ELUs alleviate the vanishing gradient problem via the identity
for positive values.
However ELUs have improved learning characteristics compared to
the units with other activation functions.
In contrast to ReLUs, ELUs have negative values which allows them
to push mean unit activations closer to zero like
batch normalization but with lower computational complexity.
Mean shifts toward zero speed up learning by bringing
the normal gradient closer to the unit natural gradient
because of a reduced bias shift effect.
While LReLUs and PReLUs have negative values, too, they
do not ensure a noise-robust deactivation state.
ELUs saturate to a negative value with smaller inputs and thereby
decrease the forward propagated variation and information.
Therefore ELUs code the degree of presence of
particular phenomena in the input, while they
do not quantitatively model the degree of their absence.

In experiments, ELUs lead not only to faster learning, but also to
significantly better generalization performance than ReLUs and LReLUs
on networks with more than 5 layers.
On CIFAR-100 ELUs networks significantly outperform ReLU networks with batch
normalization while batch normalization does not improve ELU networks.
ELU networks are among the top 10
reported CIFAR-10 results and yield the best published result on CIFAR-100,
without resorting to multi-view evaluation or model averaging.
On ImageNet, ELU networks considerably speed up learning
compared to a ReLU network with the same architecture, obtaining
less than 10\% classification error for a single crop,
single model network.
\end{abstract}

\section{Introduction}
\label{sec:intro}

Currently the most popular activation function for neural networks is
the rectified linear unit (ReLU), which was first proposed for
restricted Boltzmann machines \citep{Nair:10} and then successfully used
for neural networks \citep{Glorot:11}.
The ReLU activation function is the identity for positive
arguments and zero otherwise.
Besides producing sparse codes,
the main advantage of ReLUs is that they alleviate the vanishing
gradient problem \citep{Hochreiter:98fuzzy, Hochreiter:2001}
since the derivative of 1 for positive
values is not contractive \citep{Glorot:11}.
However ReLUs are non-negative and, therefore, have a mean
activation larger than zero.

Units that have a non-zero mean activation act as bias for
the next layer.
If such units do not cancel each other out,
learning causes a {\em bias shift} for units in next layer.
The more the units are correlated, the higher their bias shift.
We will see that Fisher optimal learning, i.e., the natural
gradient \citep{Amari:98},
would correct for the bias shift
by adjusting the weight updates.
Thus, less bias shift brings the standard gradient closer to
the natural gradient and speeds up learning.
We aim at activation functions that push activation means closer to zero
to decrease the bias shift effect.

Centering the activations at zero
has been proposed in order to keep the off-diagonal entries of the
Fisher information matrix small \citep{Raiko:12}.
For neural network it is known that centering the activations
speeds up learning \citep{LeCun:91,LeCun:98,Schraudolph:98}.
%
``Batch normalization'' also centers activations with the goal
to counter the internal covariate shift \citep{Ioffe:15}.
Also the Projected Natural Gradient Descent algorithm (PRONG)
centers the activations by implicitly whitening them
\citep{Desjardins:15}.

An alternative to centering is to push the mean activation toward zero by
an appropriate activation function.
Therefore $tanh$ has been preferred over
logistic functions \citep{LeCun:91,LeCun:98}.
Recently ``Leaky ReLUs'' (LReLUs)
that replace the negative part of the ReLU with a linear function
have been shown to be superior to ReLUs \citep{Maas:13}.
Parametric Rectified Linear Units (PReLUs) generalize LReLUs
by learning the slope of the negative part which yielded
improved learning behavior on large image benchmark data sets \citep{He:15}.
Another variant are
Randomized Leaky Rectified Linear Units (RReLUs) which randomly sample
the slope of the negative part which raised the performance on
image benchmark datasets and convolutional networks \citep{Xu:15}.

In contrast to ReLUs, activation functions like
LReLUs, PReLUs, and RReLUs do not
ensure a noise-robust deactivation state.
We propose an activation function that has negative
values to allow for mean activations close to zero,
but which saturates to a negative value with smaller arguments.
The saturation decreases the variation of the units if
deactivated, so the precise deactivation argument is less relevant.
Such an activation function can code the degree of presence of
particular phenomena in the input, but does
not quantitatively model the degree of their absence.
Therefore, such an activation function is more robust to noise.
Consequently, dependencies between
coding units are much easier to model and much easier to interpret
since only activated code units carry much information.
Furthermore, distinct concepts are much less likely to interfere with
such activation functions since the deactivation state is
non-informative, i.e.\ variance decreasing.

\section{Bias Shift Correction Speeds Up Learning}
\label{sec:naturalGradient}
To derive and analyze the bias shift effect mentioned in
the introduction, we utilize the natural gradient.
The natural gradient corrects the gradient direction with
the inverse Fisher information matrix and, thereby, enables
Fisher optimal learning, which ensures
the steepest descent in the Riemannian parameter manifold
and Fisher efficiency for online learning \citep{Amari:98}.
The recently introduced Hessian-Free Optimization
technique \citep{Martens:10}
and the Krylov Subspace Descent methods \citep{Vinyals:12}
use an extended Gauss-Newton approximation of the Hessian,
therefore they can be interpreted as versions of natural
gradient descent \citep{Pascanu:14}.

Since for neural networks the
Fisher information matrix is typically too expensive to compute,
different approximations
of the natural gradient have been proposed.
Topmoumoute Online natural Gradient
Algorithm (TONGA) \citep{LeRoux:08} uses a
low-rank approximation of natural gradient descent.
FActorized Natural Gradient (FANG) \citep{Grosse:15}
estimates the natural gradient via an approximation of
the Fisher information matrix by a Gaussian graphical model.
The Fisher information matrix can be approximated by a block-diagonal matrix,
where unit or quasi-diagonal natural gradients are used \citep{Olivier:13}.
Unit natural gradients or ``Unitwise Fisher's scoring'' \citep{Kurita:93}
are based on natural gradients for perceptrons \citep{Amari:98,Yang:98}.
We will base our analysis on the unit natural gradient.

We assume a parameterized probabilistic model $p(\Bx;\Bw)$
with parameter vector $\Bw$ and data $\Bx$.
The training data are $\BX=(\Bx_1,\ldots,\Bx_N) \in
\bbR^{(d+1)\times N}$ with
$\Bx_n = (\Bz_n^T,y_n)^T \in \bbR^{d+1}$, where $\Bz_n$ is the
input for example $n$ and $y_n$ is its label.
$L(p(.;\Bw),\Bx)$ is the loss of example $\Bx=(\Bz^T,y)^T$ using model
$p(.;\Bw)$. The average loss on the
training data $\BX$ is the empirical risk $R_{\mathrm{emp}}(p(.;\Bw),\BX)$.
Gradient descent updates the
weight vector $\Bw$ by
$\Bw^{\mathrm{new}} =  \Bw^{\mathrm{old}} - \eta \nabla_{\Bw} R_{\mathrm{emp}}$
where $\eta$ is the learning rate.
The {\em natural gradient} is the inverse Fisher
information matrix $\tilde{\BF}^{-1}$ multiplied by the gradient of the empirical
risk:
$\nabla_{\Bw}^{\mathrm{nat}} R_{\mathrm{emp}}  =  \tilde{\BF}^{-1} \nabla_{\Bw} R_{\mathrm{emp}}$.
For a  multi-layer perceptron
$\Ba$ is the unit activation vector and $a_0=1$ is the
bias unit activation.
We consider the ingoing weights to unit $i$, therefore we drop the
index $i$: $w_j=w_{ij}$ for the weight from unit $j$
to unit $i$, $a=a_i$ for the activation, and $w_{0}$ for the bias weight of unit $i$.
The activation function $f$ maps
the net input $\mathrm{net}=\sum_{j} w_{j} a_j$ of unit $i$ to its activation
$a =  f(\mathrm{net})$.
For computing the Fisher information matrix, the derivative of
the log-output probability
$\frac{\partial }{\partial w_{j}} \ln p(\Bz; \Bw)$ is required.
Therefore we define the $\delta$ at unit $i$ as
$\delta = \frac{\partial }{\partial
\mathrm{net}}  \ln p(\Bz; \Bw)$,
which can be computed via backpropagation, but using
the log-output probability instead of the conventional loss function.
The derivative is
$\frac{\partial }{\partial w_{j}} \ln
p(\Bz; \Bw ) =  \delta  a_j$.

We restrict the Fisher information matrix to weights leading to
unit $i$ which is the {\em unit Fisher information matrix} $\BF$.
$\BF$ captures only the interactions of weights to
unit $i$. Consequently, the unit natural gradient only corrects
the interactions of weights to unit $i$, i.e.\ considers the
Riemannian parameter manifold only in a subspace.
The unit Fisher information matrix is
\begin{align}
&\left[\BF(\Bw)\right]_{kj} \ = \  \EXP_{p(\Bz;\Bw)} \! \left(
 \frac{\partial \ln p(\Bz;\Bw) }{\partial w_{k}}
\ \frac{\partial \ln p(\Bz;\Bw) }{\partial w_{j}} \right)
 \ = \ \EXP_{p(\Bz;\Bw)} (\delta^2 \ a_k \ a_j )  \ .
\end{align}
Weighting the activations by
$\delta^2$ is equivalent to
adjusting the probability of drawing inputs $\Bz$. Inputs $\Bz$ with large
$\delta^2$ are drawn with higher probability.
Since $0 \leq \delta^2=\delta^2(\Bz)$, we can define a
distribution  $q(\Bz)$:
\begin{align}
q(\Bz) \ &= \ \delta^2(\Bz) \
  p(\Bz) \ \left(\int \delta^2(\Bz) \ p(\Bz) \ d\Bz \right)^{-1}
\ = \ \delta^2(\Bz) \
  p(\Bz) \ \EXP_{p(\Bz)}^{-1} (\delta^2 )  \ .
\end{align}
Using $q(\Bz)$, the entries of $\BF$ can be expressed as second moments:
\begin{align}
\left[\BF(\Bw)\right]_{kj} \ &= \ \EXP_{p(\Bz)} (
  \delta^2  \ a_k \ a_j )
\ = \
 \int  \delta^2  \ a_k \ a_j \ p(\Bz) \ d\Bz
\ = \  \EXP_{p(\Bz)} (\delta^2 ) \ \EXP_{q(\Bz)}( a_k \ a_j ) \ .
\end{align}

If the bias unit is $a_0=1$ with weight $w_0$
then the weight vector can be divided
into a bias part $w_0$ and the rest $\Bw$: $(\Bw^T,w_0)^T$.
For the row $\Bb = \left[\BF(\Bw)\right]_{0}$
that corresponds to the bias weight, we have:
\begin{align}
\label{eq:b}
\Bb \ &= \ \EXP_{p(\Bz)} (\delta^2 \Ba )
\ = \ \EXP_{p(\Bz)} (\delta^2 ) \ \EXP_{q(\Bz)}
(\Ba) \ = \  \COV_{p(\Bz)} ( \delta^2 , \Ba ) \ + \
\EXP_{p(\Bz)} (\Ba) \ \EXP_{p(\Bz)} ( \delta^2 )
\ .
\end{align}
The next Theorem~\ref{th:th1} gives the correction of the standard
gradient by the unit natural gradient where
the bias weight is treated separately (see also \citet{Yang:98}).
\begin{theorem}
\label{th:th1}
The unit natural gradient corrects the weight update
$(\Delta \Bw^T , \Delta w_0)^T$ to a unit $i$ by
following affine transformation of
the gradient  $\nabla_{(\Bw^T,w_0)^T}  R_{\mathrm{emp}} =(\Bg^T , g_0)^T$:
\begin{align}
\begin{pmatrix}
\Delta \Bw \\
\Delta w_0
\end{pmatrix}
\ &= \
\begin{pmatrix}
\BA^{-1} \ \left( \Bg \ - \ \Delta w_0 \ \Bb \right) \\
s \ \left( g_0 \ - \ \Bb^T \ \BA^{-1} \Bg \right)
\end{pmatrix} \ ,
\end{align}
where $\BA= \left[\BF(\Bw)\right]_{\neg 0,\neg 0} =
\EXP_{p(\Bz)}( \delta^2 )  \EXP_{q(\Bz)}(\Ba \Ba^T)$
is the unit Fisher information matrix without row 0 and column 0
corresponding to the bias weight.
The vector $\Bb=\left[\BF(\Bw)\right]_{0}$ is the zeroth column of
$\BF$ corresponding to the bias weight,
and the positive scalar $s$ is
\begin{align}
s \ &= \     \EXP_{p(\Bz)}^{-1} ( \delta^2 ) \ \left( 1 \ + \
\EXP_{q(\Bz)}^T(\Ba)  \ \VAR_{q(\Bz)}^{-1} (\Ba) \ \EXP_{q(\Bz)} (\Ba) \right) \ ,
\end{align}
where $\Ba$ is the vector of activations of units with weights to unit $i$ and
$q(\Bz) =  \delta^2(\Bz) p(\Bz)
\EXP_{p(\Bz)}^{-1} (\delta^2 )$.
\end{theorem}

\begin{proof}
Multiplying the inverse Fisher matrix $\BF^{-1}$ with the separated
gradient $\nabla_{(\Bw^T,w_0)^T}  R_{\mathrm{emp}}((\Bw^T,w_0)^T,\BX) =(\Bg^T , g_0)^T$
gives the weight update $(\Delta \Bw^T , \Delta w_0)^T$:
\begin{align}
\begin{pmatrix}
\Delta \Bw \\
\Delta w_0
\end{pmatrix}
\ &= \
\begin{pmatrix}
\BA & \Bb \\
\Bb^T & c
\end{pmatrix}^{-1} \
\begin{pmatrix}
\Bg \\
g_0
\end{pmatrix}
\ = \
\begin{pmatrix}
\BA^{-1} \ \Bg \ + \ \Bu \ s^{-1} \Bu^T \Bg \ + \ g_o \ \Bu \\
 \Bu^T \Bg \ + \ s \ g_0
\end{pmatrix} \ .
\end{align}
where
\begin{align}
\Bb \ &= \ \left[\BF(\Bw)\right]_{0} \quad , \quad
c  \ = \ \left[\BF(\Bw)\right]_{00} \quad , \quad \Bu \ = \ - \  s \
\BA^{-1} \ \Bb \quad , \quad s \ = \ \left( c \ - \ \Bb^T\BA^{-1} \Bb \right)^{-1} \ .
\end{align}

The previous formula is derived in
Lemma~\ref{th:lemma1} in the appendix.
Using $\Delta w_0$ in the update gives
\begin{align}
\label{eq:updateW0}
\begin{pmatrix}
\Delta \Bw \\
\Delta w_0
\end{pmatrix}
\ &= \
\begin{pmatrix}
\BA^{-1} \ \Bg \ + \ s^{-1} \Bu \ \Delta w_0 \\
 \Bu^T \Bg \ + \ s \ g_0
\end{pmatrix} \quad , \quad
\begin{pmatrix}
\Delta \Bw \\
\Delta w_0
\end{pmatrix}
\ = \
\begin{pmatrix}
\BA^{-1} \ \left( \Bg \ - \ \Delta w_0 \ \Bb \right) \\
s \ \left( g_0 \ - \ \Bb^T \ \BA^{-1} \Bg \right)
\end{pmatrix} \ .
\end{align}
The right hand side is obtained by inserting
$\Bu=  - s \BA^{-1} \Bb$ in the left hand side
update.
Since $c=F_{00}=\EXP_{p(\Bz)} ( \delta^2 )$,
$\Bb= \EXP_{p(\Bz)} ( \delta^2 )  \EXP_{q(\Bz)} (\Ba)$,
and $\BA= \EXP_{p(\Bz)} ( \delta^2 )  \EXP_{q(\Bz)}(\Ba \Ba^T)$, we obtain
\begin{align}
s \ &= \ \left( c \ - \ \Bb^T\BA^{-1} \Bb \right)^{-1}
\ = \ \EXP_{p(\Bz)}^{-1} ( \delta^2 ) \
\left( 1 \ - \ \EXP_{q(\Bz)}^T(\Ba)  \
\EXP_{q(\Bz)}^{-1} (\Ba \Ba^T) \ \EXP_{q(\Bz)} (\Ba) \right)^{-1}  \ .
\end{align}
Applying Lemma~\ref{th:lemma2} in the appendix gives the formula for $s$.
\end{proof}

The {\em bias shift} (mean shift) of unit $i$ is the change of unit $i$'s
mean value due to the weight update. Bias shifts of unit $i$
lead to oscillations and impede learning. See Section~4.4 in \citet{LeCun:98} for
demonstrating this effect at the inputs and in \citet{LeCun:91} for explaining
this effect using the input covariance matrix.
Such bias shifts are mitigated or even prevented
by the unit natural gradient.
The {\em bias shift correction} of the unit natural gradient is the
effect on the bias shift due to $\Bb$ which captures
the interaction between the bias unit and the incoming units.
Without bias shift correction, i.e., $\Bb=\BZe$ and $s=c^{-1}$,
the weight updates are $\Delta \Bw=\BA^{-1}\Bg$ and $\Delta w_0 = c^{-1} g_0$.
As only the activations
depend on the input, the bias shift can be computed by multiplying the weight
update by the mean of the activation vector $\Ba$.
Thus we obtain the bias shift
$(\EXP_{p(\Bz)} (\Ba)^T ,1) (\Delta \Bw^T, \Delta w_0)^T=
\EXP_{p(\Bz)}^T (\Ba) \BA^{-1}\Bg + c^{-1} g_0$.
The bias shift strongly depends on the correlation of the incoming
units which is captured by $\BA^{-1}$.

Next, Theorem~\ref{th:th2} states that the bias shift correction by the
unit natural gradient can be
considered to correct the incoming mean $\EXP_{p(\Bz)} (\Ba)$ proportional
to $\EXP_{q(\Bz)} (\Ba)$ toward zero.
\begin{theorem}
\label{th:th2}
The bias shift correction by the unit natural gradient is equivalent to
an additive correction of the incoming mean by
$- k \ \EXP_{q(\Bz)} (\Ba)$
and a multiplicative correction of the bias unit by $k$, where
\begin{align}
\label{eq:k}
k \ &= \ 1 \ + \  \left( \EXP_{q(\Bz)}(\Ba) \ - \ \EXP_{p(\Bz)}(\Ba) \right)^T
\VAR_{q(\Bz)}^{-1}(\Ba) \ \EXP_{q(\Bz)}(\Ba) \ .
\end{align}
\end{theorem}
\begin{proof}
Using
$\Delta w_0 = - s \Bb^T \BA^{-1} \Bg + s g_0$, the bias shift is:
\begin{align}
&\begin{pmatrix}
\EXP_{p(\Bz)} (\Ba) \\
1
\end{pmatrix}^T \
\begin{pmatrix}
\Delta \Bw \\
\Delta w_0
\end{pmatrix}
\ = \
\begin{pmatrix}
\EXP_{p(\Bz)} (\Ba) \\
1
\end{pmatrix}^T \
\begin{pmatrix}
\BA^{-1} \ \Bg \ - \  \BA^{-1} \Bb  \ \Delta w_0\\
\Delta w_0
\end{pmatrix}  \\ \nonumber
&= \ \EXP_{p(\Bz)}^T (\Ba) \ \BA^{-1} \ \Bg \ + \
\left( 1 \ - \ \EXP_{p(\Bz)}^T (\Ba) \ \BA^{-1} \Bb \right) \
\Delta w_0  \\ \nonumber
&= \ \left( \EXP_{p(\Bz)}^T (\Ba) \ - \
\underbrace{\left( 1 \ - \ \EXP_{p(\Bz)}^T (\Ba) \ \BA^{-1} \Bb \right) \ s \
\Bb^T } \right) \BA^{-1} \ \Bg \ + \  s \ \left( 1 \ - \
\EXP_{p(\Bz)}^T (\Ba) \ \BA^{-1} \Bb \right) \ g_0 \ .
\end{align}

The mean correction term, indicated by an underbrace in previous
formula, is
\begin{align}
&s \ \left( 1 \ - \ \EXP_{p(\Bz)}^T (\Ba) \ \BA^{-1} \Bb \right) \ \Bb \
=  \ \EXP_{p(\Bz)}^{-1} ( \delta^2 ) \ \left( 1 \ - \ \EXP_{q(\Bz)}^T
(\Ba)  \ \EXP_{q(\Bz)}^{-1} (\Ba \Ba^T) \ \EXP_{q(\Bz)}
(\Ba) \right)^{-1}\\\nonumber
&~~~~~~\left( 1 \ - \ \EXP_{p(\Bz)}^T (\Ba) \
\EXP_{q(\Bz)}^{-1} (\Ba \Ba^T)  \ \EXP_{q(\Bz)} (\Ba) \right) \
\EXP_{p(\Bz)} ( \delta^2 ) \ \EXP_{q(\Bz)} (\Ba) \\\nonumber
&= \ \underbrace{\left( 1  -  \EXP_{q(\Bz)}^T(\Ba)  \
\EXP_{q(\Bz)}^{-1} (\Ba \Ba^T) \ \EXP_{q(\Bz)}(\Ba) \right)^{-1}
\left( 1  -  \EXP_{p(\Bz)}^T (\Ba) \ \EXP_{q(\Bz)}^{-1} (\Ba \Ba^T)
\ \EXP_{q(\Bz)}(\Ba) \right)}_{k}  \EXP_{q(\Bz)}(\Ba) .
\end{align}
The expression Eq.~\eqref{eq:k} for $k$ follows from
Lemma~\ref{th:lemma2} in the appendix.
The bias unit correction term is
$s \left( 1 - \EXP_{p(\Bz)}^T (\Ba) \BA^{-1} \Bb \right) g_0 =
k c^{-1} g_0$.
\end{proof}

In Theorem~\ref{th:th2} we can reformulate
$k =  1  +   \EXP_{p(\Bz)}^{-1} ( \delta^2 ) \COV_{p(\Bz)}^T (  \delta^2 , \Ba )
\VAR_{q(\Bz)}^{-1}(\Ba) \EXP_{q(\Bz)}(\Ba)$. Therefore $k$ increases with the length of
$\EXP_{q(\Bz)}(\Ba)$ for given variances and covariances.
Consequently the bias shift correction through the
unit natural gradient is governed by the length of $\EXP_{q(\Bz)}(\Ba)$.
The bias shift correction is zero for $\EXP_{q(\Bz)} (\Ba)=\BZe$ since
$k=1$ does not correct the bias unit multiplicatively.
Using Eq.~\eqref{eq:b}, $\EXP_{q(\Bz)}(\Ba)$ is split into an
offset and an information containing term:
\begin{align}
\EXP_{q(\Bz)} (\Ba) \ =  \ \EXP_{p(\Bz)} (\Ba) \ + \
\EXP_{p(\Bz)}^{-1} ( \delta^2 ) \
\COV_{p(\Bz)} ( \delta^2 , \Ba ) \ .
\end{align}
In general, {\em smaller positive $\EXP_{p(\Bz)} (\Ba)$ lead to smaller
positive $\EXP_{q(\Bz)}(\Ba)$, therefore to smaller corrections.}
The reason is that in general the largest absolute components
of $\COV_{p(\Bz)} (\delta^2, \Ba )$ are positive,
since activated inputs will activate the unit $i$ which in turn will have
large impact on the output.

To summarize,
the unit natural gradient corrects the bias shift of unit $i$
via the interactions of incoming units with the bias unit
to ensure efficient learning.
This correction is equivalent to shifting the mean activations
of the incoming units toward zero and scaling up the bias unit.
To reduce the undesired bias shift effect without the natural gradient,
either the
(i) activation of incoming units can be
centered at zero
or (ii) activation functions with negative values can be used.
We introduce a new activation function with negative values while
keeping the identity for positive arguments where it is not contradicting.

\section{Exponential Linear Units (ELUs)}
\label{sec:elu}


The {\em exponential linear unit} (ELU) with $0<\alpha$ is
\begin{align}
\label{eq:elu}
&f(x) \ = \
\begin{cases}
x &\mbox{if } x > 0 \\
\alpha \ (\exp(x)-1) & \mbox{if } x \leq 0
\end{cases} \quad ,  \quad
f'(x) \ = \
\begin{cases}
1 &\mbox{if } x > 0 \\
f(x) + \alpha & \mbox{if } x \leq 0
\end{cases} \ .
\end{align}
The ELU hyperparameter $\alpha$ controls the value to
which an ELU saturates for negative net inputs (see Fig.~\ref{fig:ActivationFunction}).
ELUs diminish the vanishing gradient effect
as rectified linear units (ReLUs) and leaky ReLUs
(LReLUs) do.
The vanishing gradient problem is alleviated because
the positive part of these functions is the identity,
therefore their derivative is one and not contractive.
In contrast, $\tanh$ and sigmoid activation functions are contractive
almost everywhere.

\begin{wrapfigure}{r}{0.45\textwidth}
\vspace*{-10pt}
\begin{center}
\includegraphics[width=0.45\textwidth]{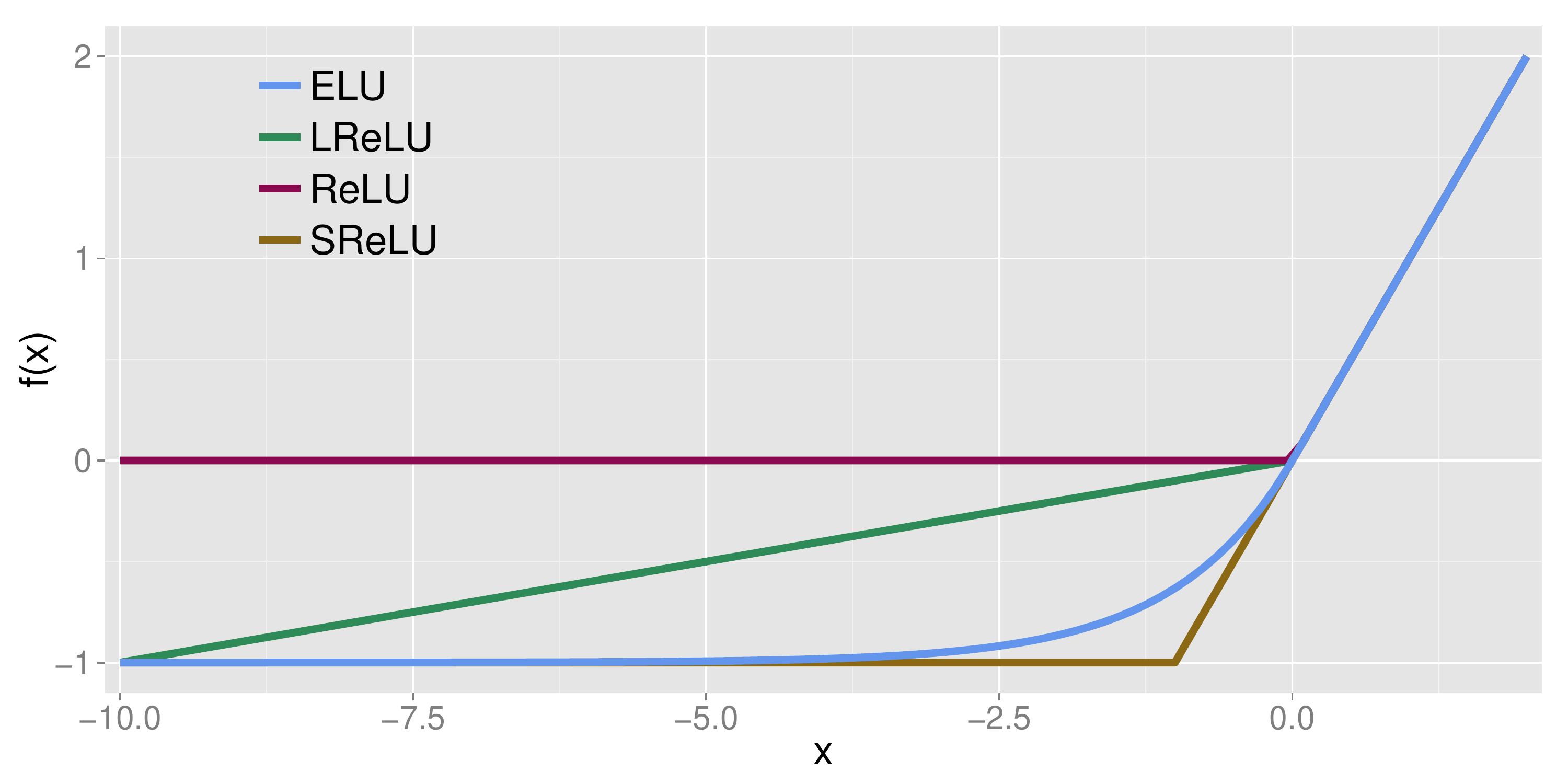}
\end{center}
\caption{The rectified linear unit (ReLU), the leaky ReLU (LReLU, $\alpha=0.1$), the shifted ReLUs (SReLUs),
and the exponential linear unit (ELU, $\alpha=1.0$). \label{fig:ActivationFunction}}
\vspace*{-5pt}
\end{wrapfigure}
In contrast to ReLUs, ELUs have negative values which
pushes the mean of the activations closer to zero.
Mean activations that are closer to zero enable faster learning as
they bring the gradient closer to the natural gradient (see
Theorem~\ref{th:th2} and text thereafter).
ELUs saturate to a negative value when the argument gets smaller.
Saturation means a small derivative
which decreases the variation and the information that is
propagated to the next layer. Therefore the representation
is both noise-robust and low-complex \citep{Hochreiter:99nc}.
ELUs code the degree of presence of input concepts,
while they neither quantify the degree of their absence nor distinguish
the causes of their absence.
This property of non-informative deactivation states is also present at ReLUs
and allowed to detect biclusters corresponding to
biological modules in gene expression datasets \citep{Clevert:15nips} and
to identify toxicophores in toxicity prediction \citep{Unterthiner:15,Mayr:15}.
The enabling features for these interpretations
is that activation can be clearly distinguished from deactivation
and that only active units carry relevant information and can crosstalk.

\section{Experiments Using ELUs}
\label{sec:exp}
\begin{figure}[!ht]
\begin{center}
\subfigure[Average unit activation]{
\includegraphics[angle=0,width= 0.49\textwidth]{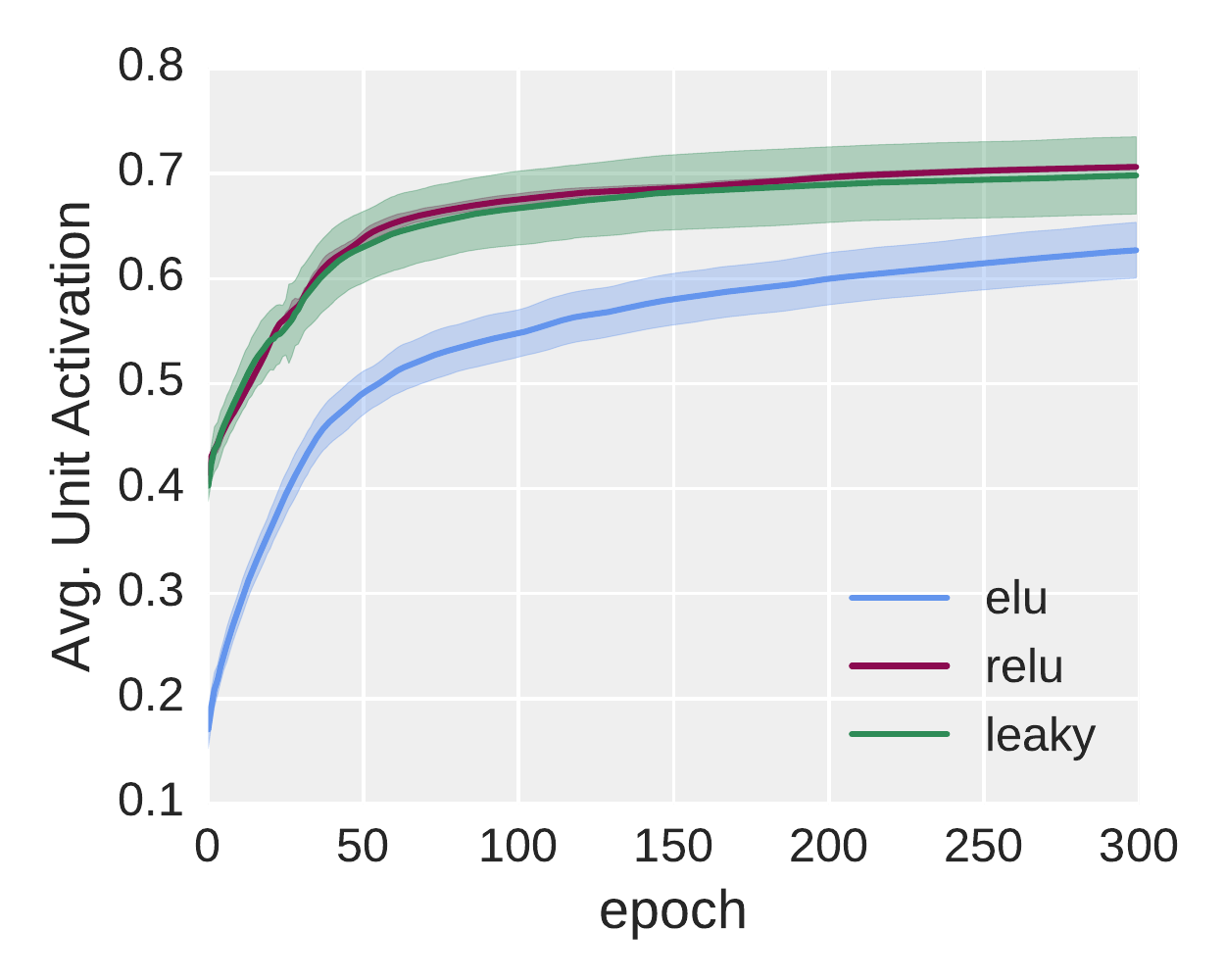}}
\subfigure[Cross entropy loss]{
\includegraphics[angle=0,width= 0.49\textwidth]{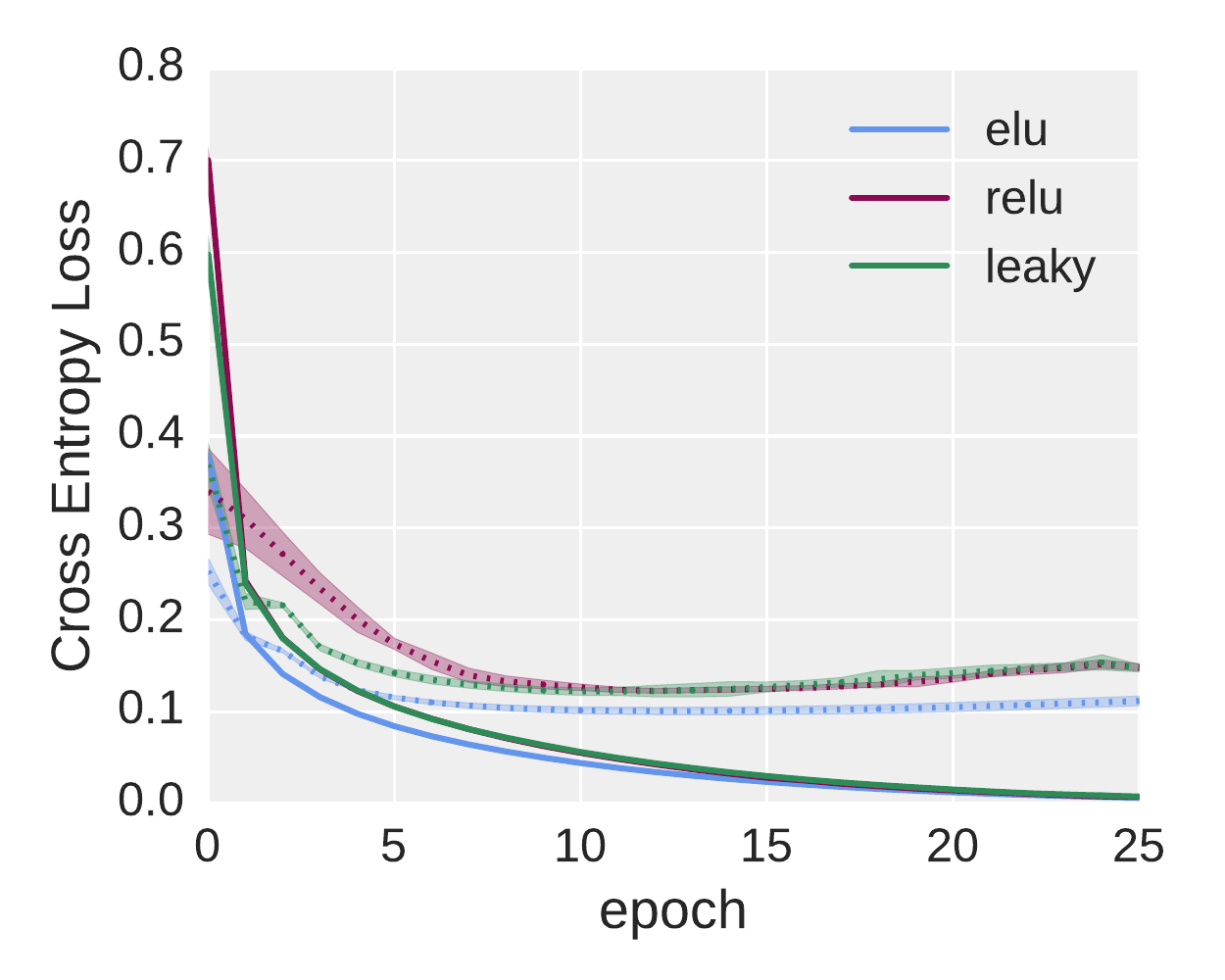}}
\end{center}
\caption{ELU networks evaluated at MNIST. Lines are
the average over five runs with different random initializations, error
bars show standard deviation.
Panel (a): median of the average unit activation for different
activation functions.
Panel (b): Training set (straight line) and validation set (dotted line)
 cross entropy loss. All lines stay flat after epoch 25.
\label{fig:mnistplots}}
\end{figure}

In this section, we assess the performance of exponential linear units (ELUs)
if used for unsupervised and supervised learning of deep autoencoders and
deep convolutional networks. ELUs with $\alpha=1.0$  are compared to
(i) Rectified Linear Units (ReLUs) with activation $f(x)=\max(0,x)$,
(ii) Leaky ReLUs (LReLUs) with activation $f(x)=\max(\alpha x,x)$ ($0<\alpha<1$), and
(iii) Shifted ReLUs (SReLUs) with activation $f(x)=\max(-1,x)$.
Comparisons are done with and without batch normalization.
The following benchmark datasets are used:
(i) {\em MNIST} (gray images in 10 classes, 60k train and 10k test),
(ii) {\em CIFAR-10} (color images in 10 classes, 50k train and 10k test),
(iii) {\em CIFAR-100} (color images in 100 classes, 50k train and 10k test), and
(iv) {\em ImageNet} (color images in 1,000 classes, 1.3M train and 100k tests).

\subsection{MNIST}

\subsubsection{Learning Behavior}
We first want to verify that ELUs keep the mean activations closer to
zero than other units. Fully connected deep neural networks
with ELUs ($\alpha=1.0$), ReLUs, and LReLUs ($\alpha=0.1$) were
trained on the MNIST digit classification
dataset while each hidden unit's activation was tracked.
Each network had eight hidden layers of 128 units each, and was trained
for 300 epochs by
stochastic gradient descent with
learning rate $0.01$ and mini-batches of size 64.
The weights have been initialized according to \citep{He:15}.
After each epoch we calculated the units' average activations on a fixed
subset of the training data.
Fig.~\ref{fig:mnistplots} shows the median over all units along learning.
ELUs stay have smaller median throughout the training process.
The training error of ELU networks decreases much more rapidly than for the other
networks.

Section~\ref{sec:varianceMean} in the appendix compares the variance
of median activation in ReLU and ELU networks.
The median varies much more in ReLU networks.
This indicates that ReLU networks continuously try to correct
the bias shift introduced by previous weight updates
while this effect is much less prominent in ELU networks.

\subsubsection{Autoencoder Learning}

\begin{figure}[!ht]
\begin{center}
\subfigure[Training set]{
\includegraphics[angle=0,width= 0.49\textwidth]{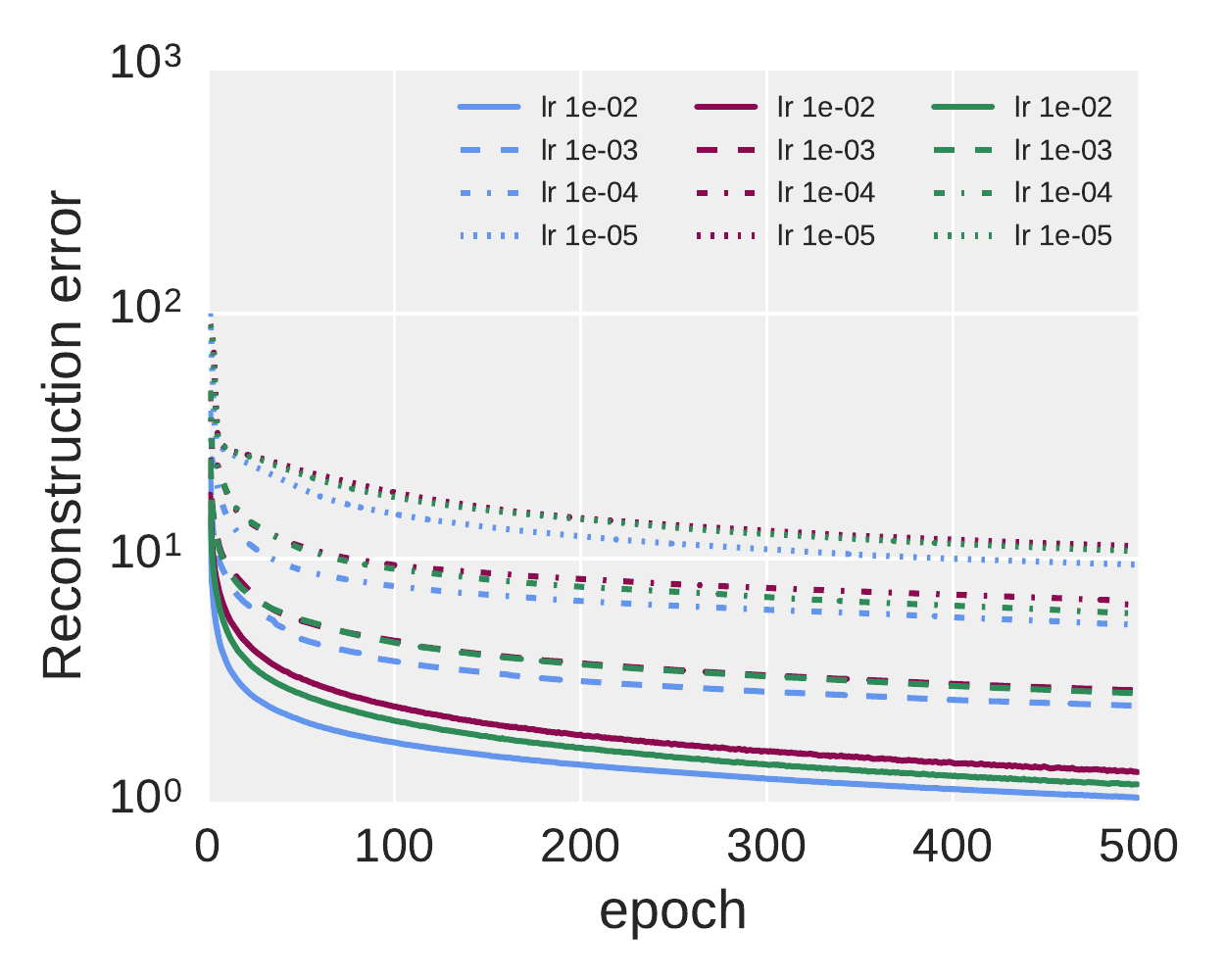}}
\subfigure[Test set]{
\includegraphics[angle=0,width= 0.49\textwidth]{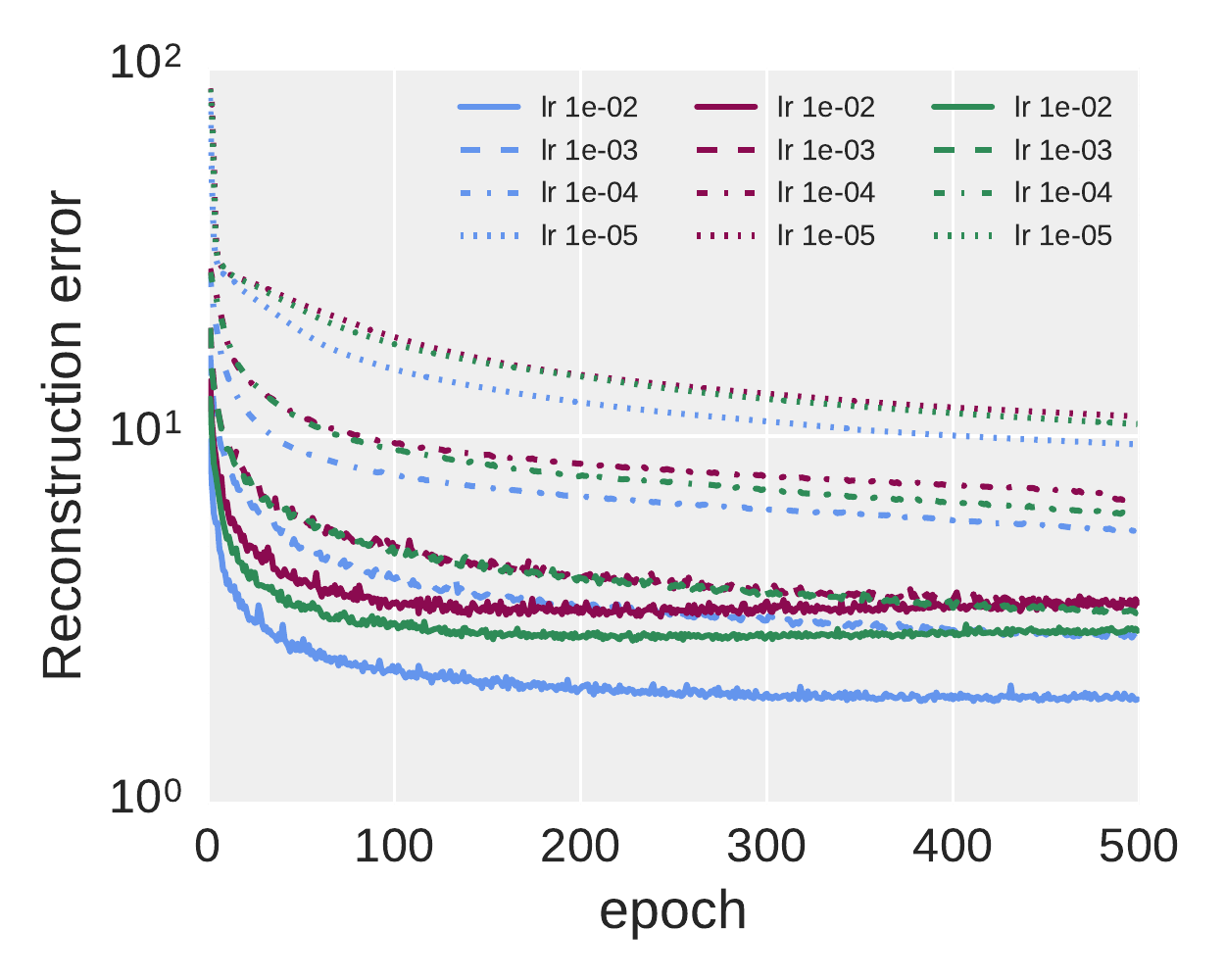}}
\end{center}
\caption{Autoencoder training on MNIST: Reconstruction error for the test and training  data set over
epochs, using different activation functions and learning rates. The results are medians
over several runs with different random initializations.
\label{fig:mnistAuto}}
\end{figure}

To evaluate ELU networks at unsupervised settings, we followed
\citet{Martens:10} and \citet{Desjardins:15} and trained a
deep autoencoder on the MNIST dataset.
The encoder part consisted of four
fully connected hidden layers with sizes 1000, 500, 250 and 30, respectively.
The decoder part was symmetrical to the encoder.
For learning we applied stochastic gradient descent with mini-batches of 64 samples for 500 epochs
using the fixed learning rates ($10^{-2}, 10^{-3}, 10^{-4}, 10^{-5}$).
Fig.~\ref{fig:mnistAuto} shows, that ELUs outperform the competing
activation functions in terms of training / test set reconstruction error for all learning rates.
As already noted by \citet{Desjardins:15}, higher learning rates
seem to perform better.

\subsection{Comparison of Activation Functions}
\label{sec:compFunctions}

In this subsection we show that ELUs indeed possess a superior
learning behavior compared to other activation functions as
postulated in Section~\ref{sec:elu}.
Furthermore we show that ELU networks perform better than ReLU
networks with batch normalization.
We use as benchmark dataset CIFAR-100 and use a relatively
simple convolutional neural network (CNN) architecture to keep the
computational complexity reasonable for comparisons.

\begin{figure*}[!ht]
\begin{center}
\subfigure[Training loss]{
\includegraphics[angle=0,width= 0.32\textwidth]{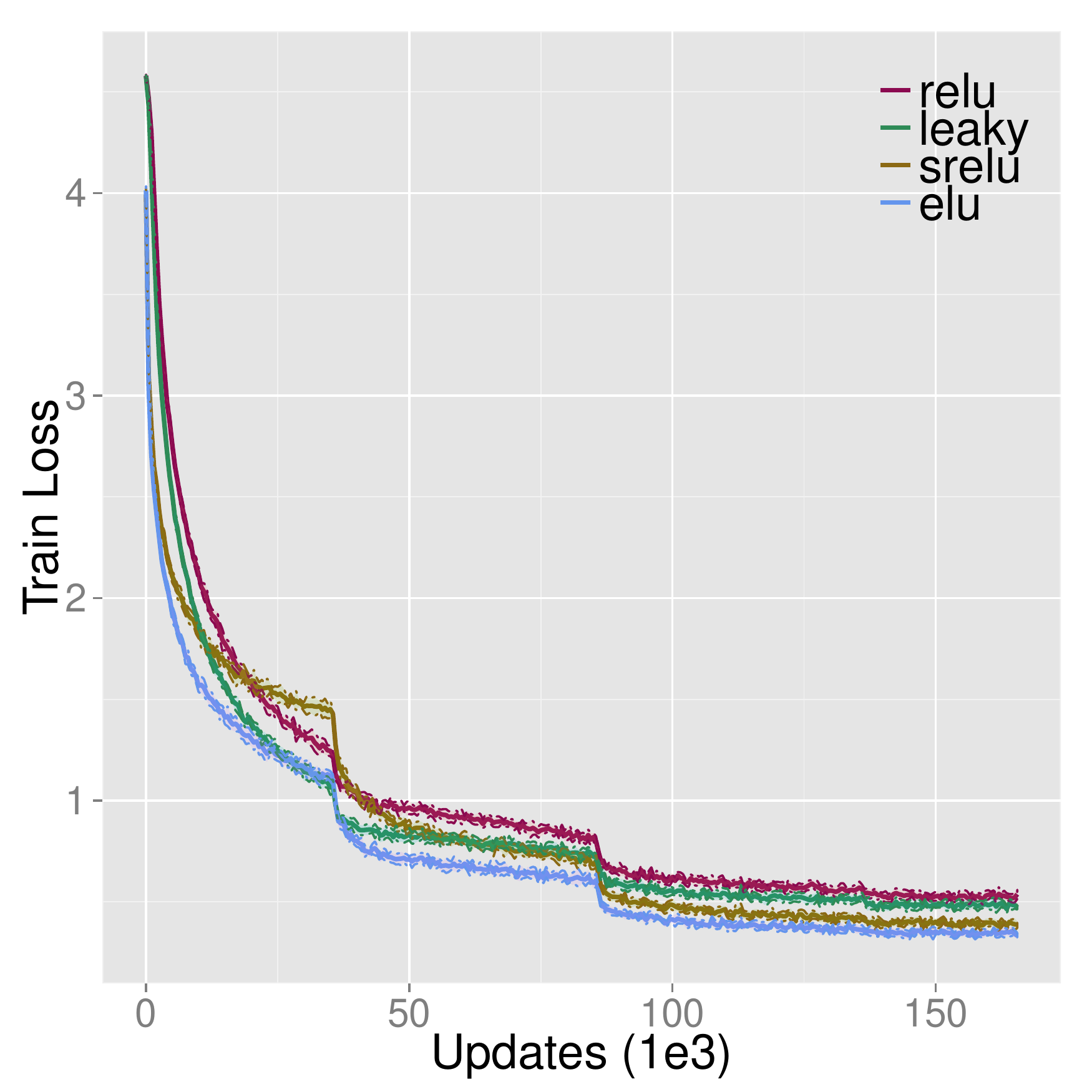}}
\subfigure[Training loss (start)]{
\includegraphics[angle=0,width= 0.32\textwidth]{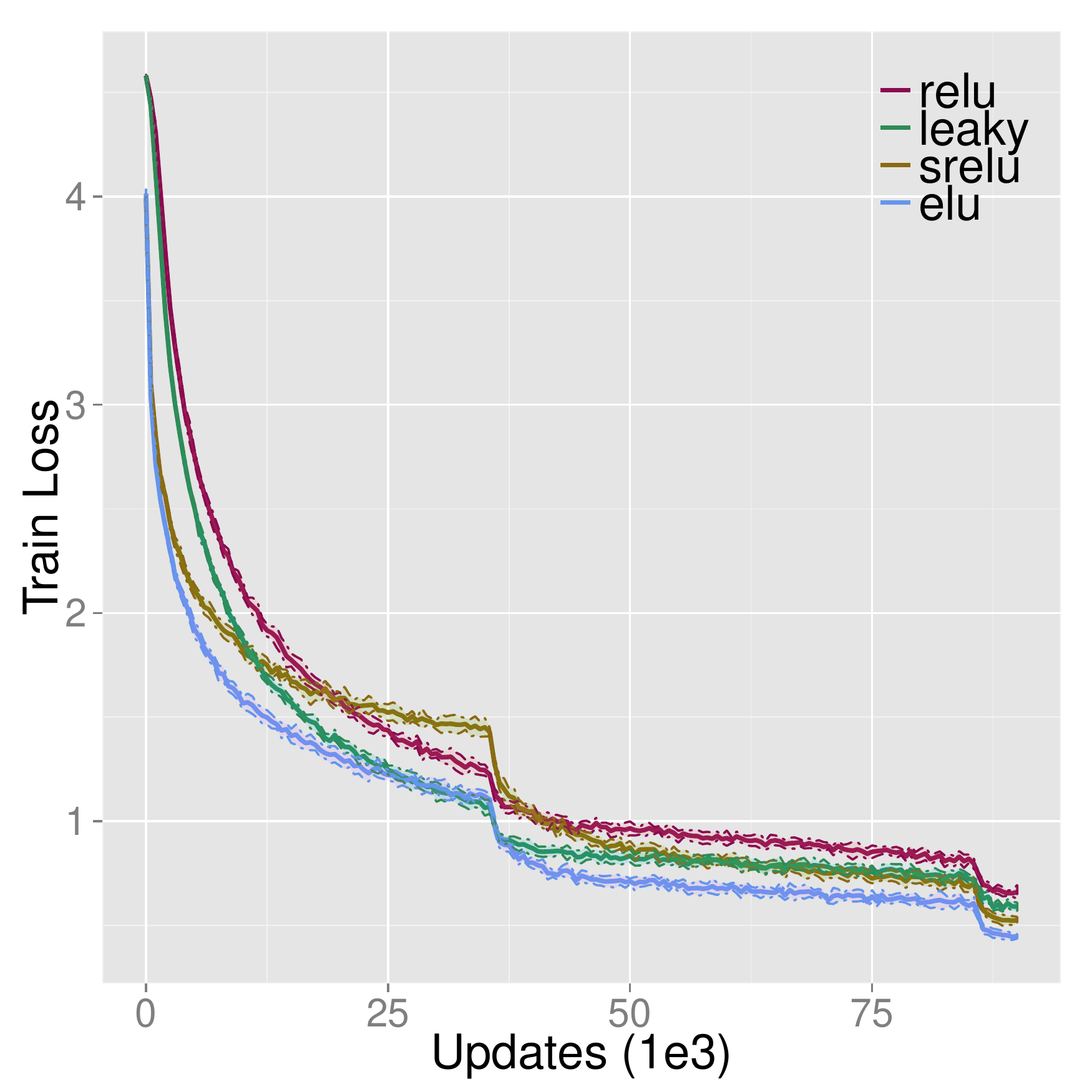}}
\subfigure[Training loss (end)]{
\includegraphics[angle=0,width= 0.32\textwidth]{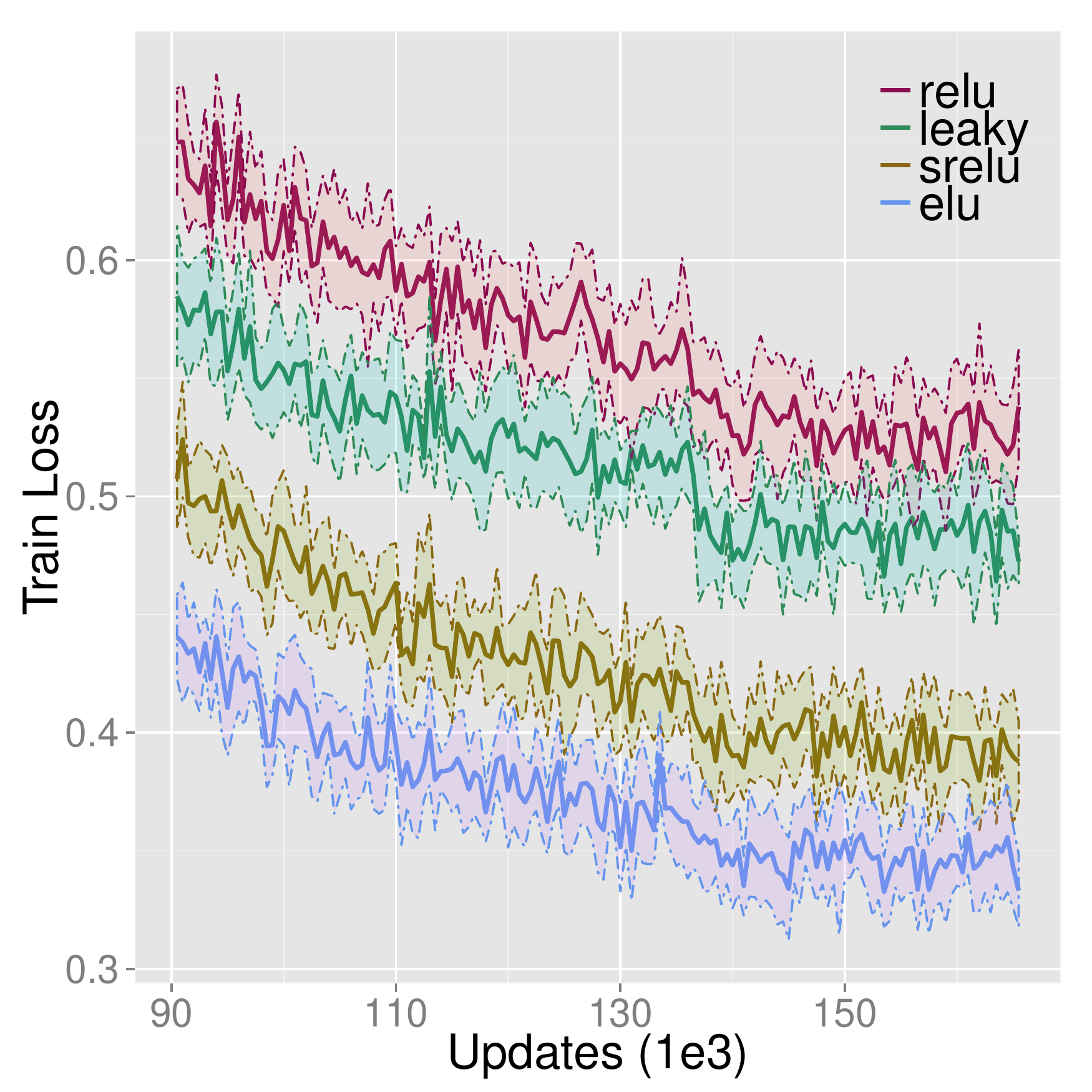}}\\[-2.0ex]
\subfigure[Test error]{
\includegraphics[angle=0,width= 0.32\textwidth]{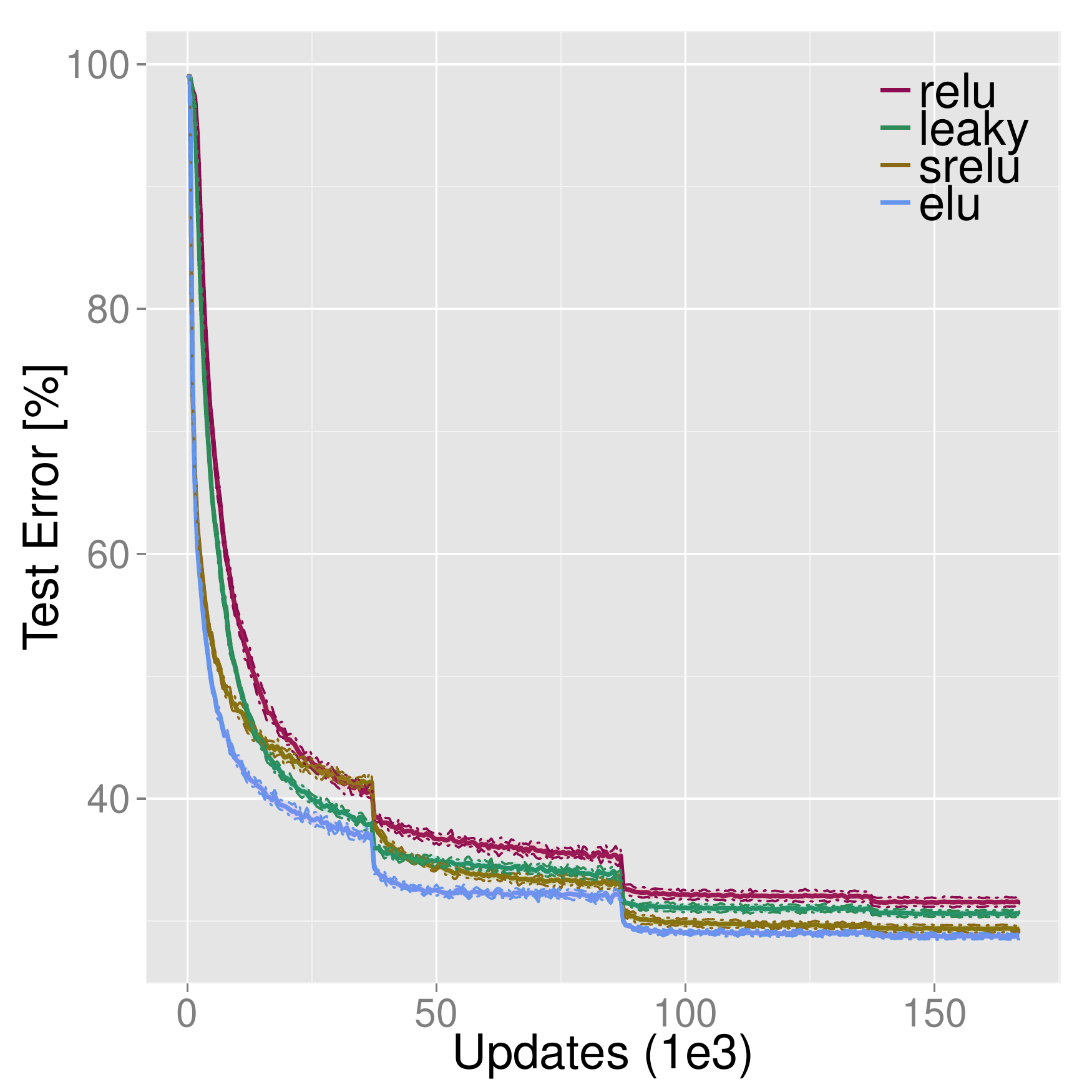}}
\subfigure[Test error (start)]{
\includegraphics[angle=0,width= 0.32\textwidth]{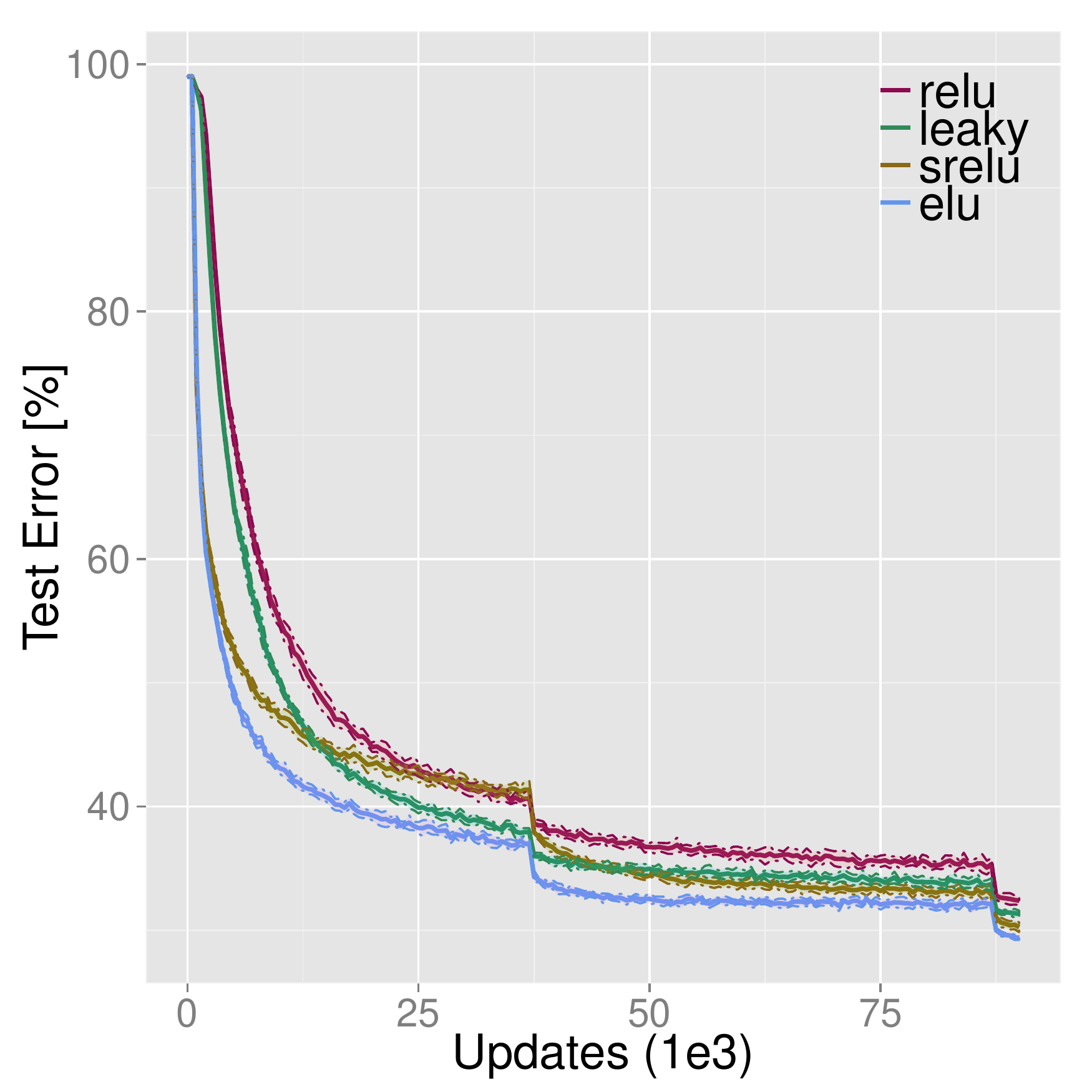}}
\subfigure[Test error (end)]{
\includegraphics[angle=0,width= 0.32\textwidth]{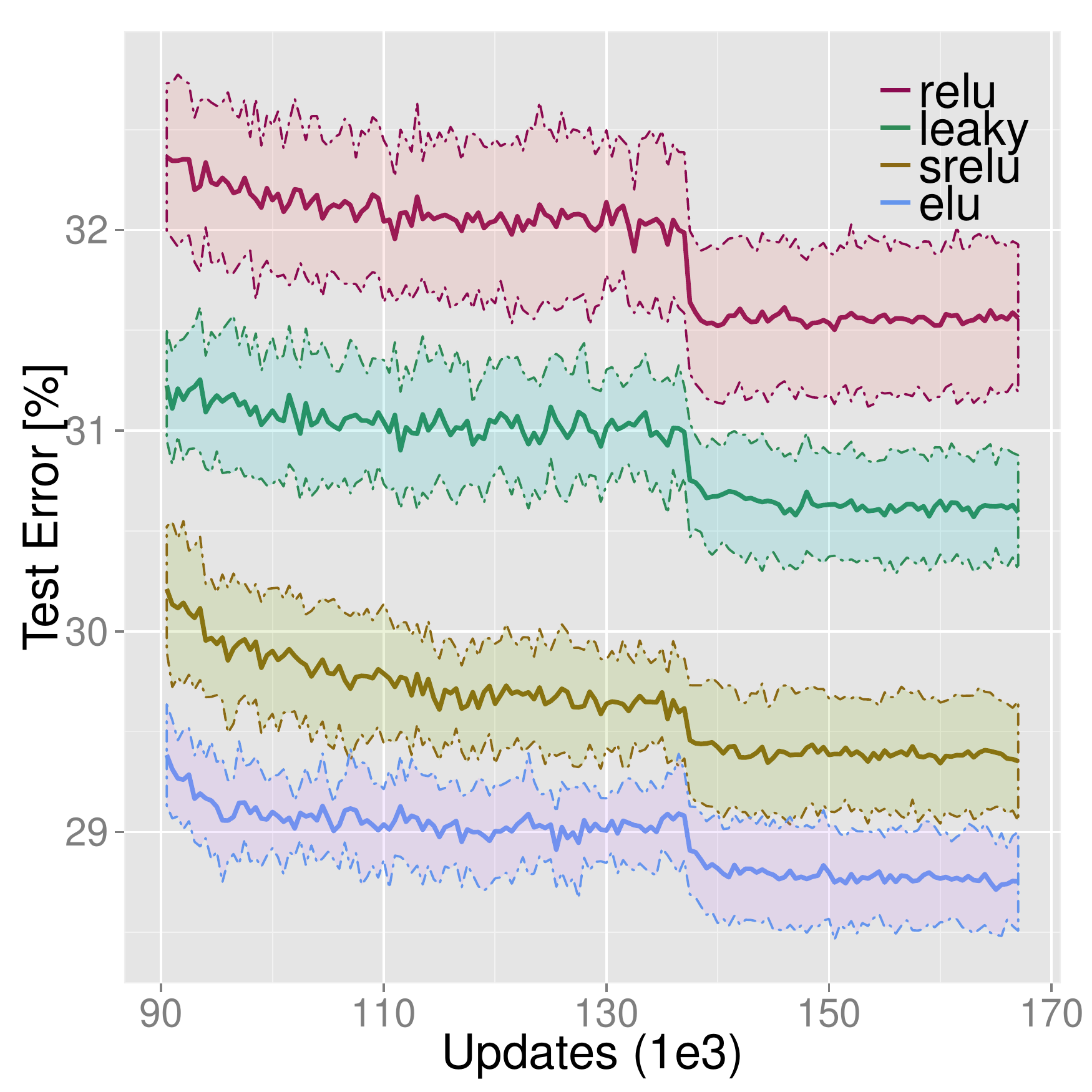}}\\[-2.0ex]
\caption{Comparison of ReLUs, LReLUs, and SReLUs on CIFAR-100.
Panels (a-c) show the training loss, panels (d-f) the test classification
error. The ribbon band show the mean and standard
deviation for 10 runs along the curve. ELU networks achieved lowest test
error and training loss.
\label{fig:resCIFAR100}}
\end{center}
\vspace*{-5pt}
\end{figure*}

\begin{figure*}[!ht]
\begin{center}
\subfigure[ELU - ReLU ]{
\includegraphics[angle=0,width= 0.32\textwidth]{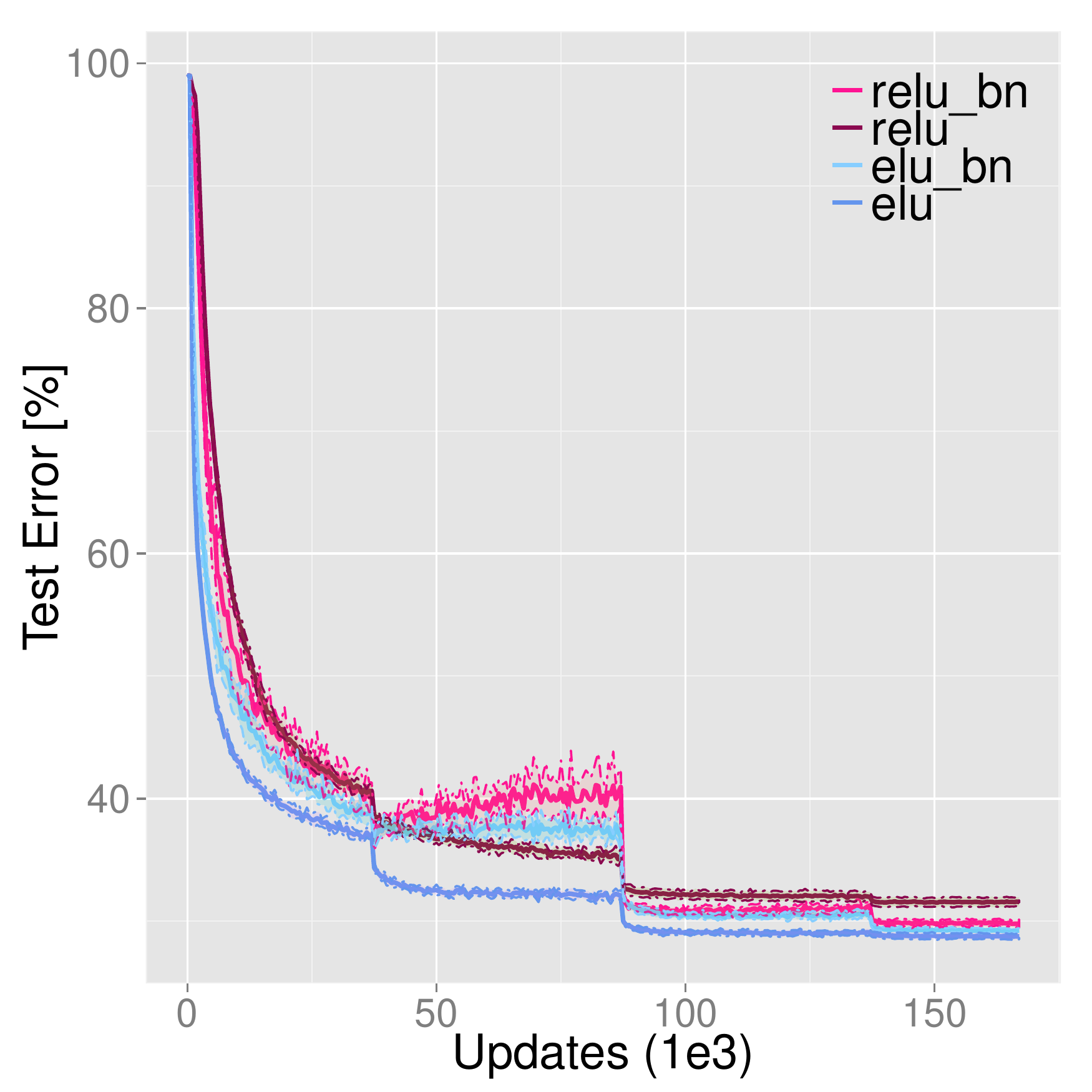}}
\subfigure[ELU - SReLU ]{
\includegraphics[angle=0,width= 0.32\textwidth]{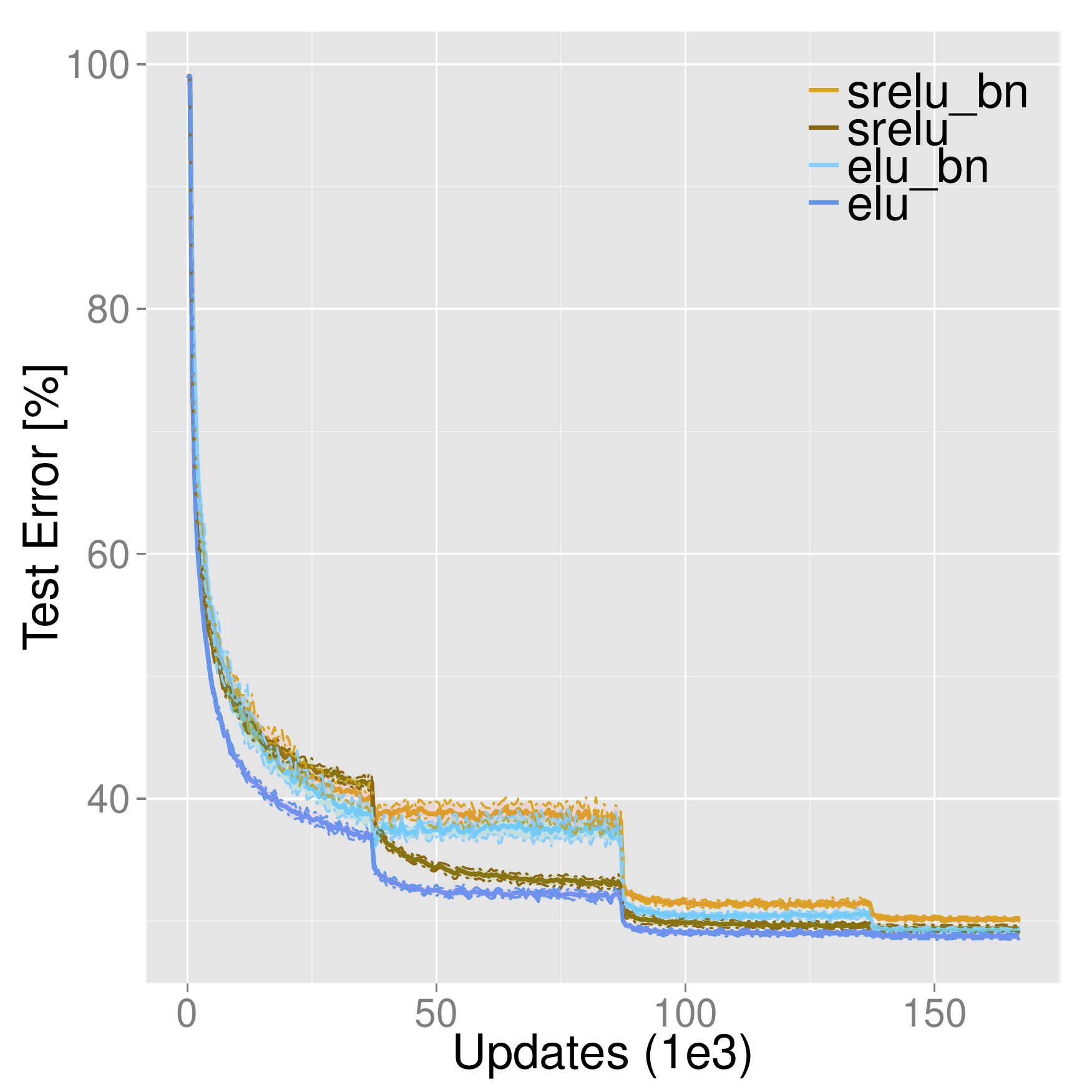}}
\subfigure[ELU - LReLU ]{
\includegraphics[angle=0,width= 0.32\textwidth]{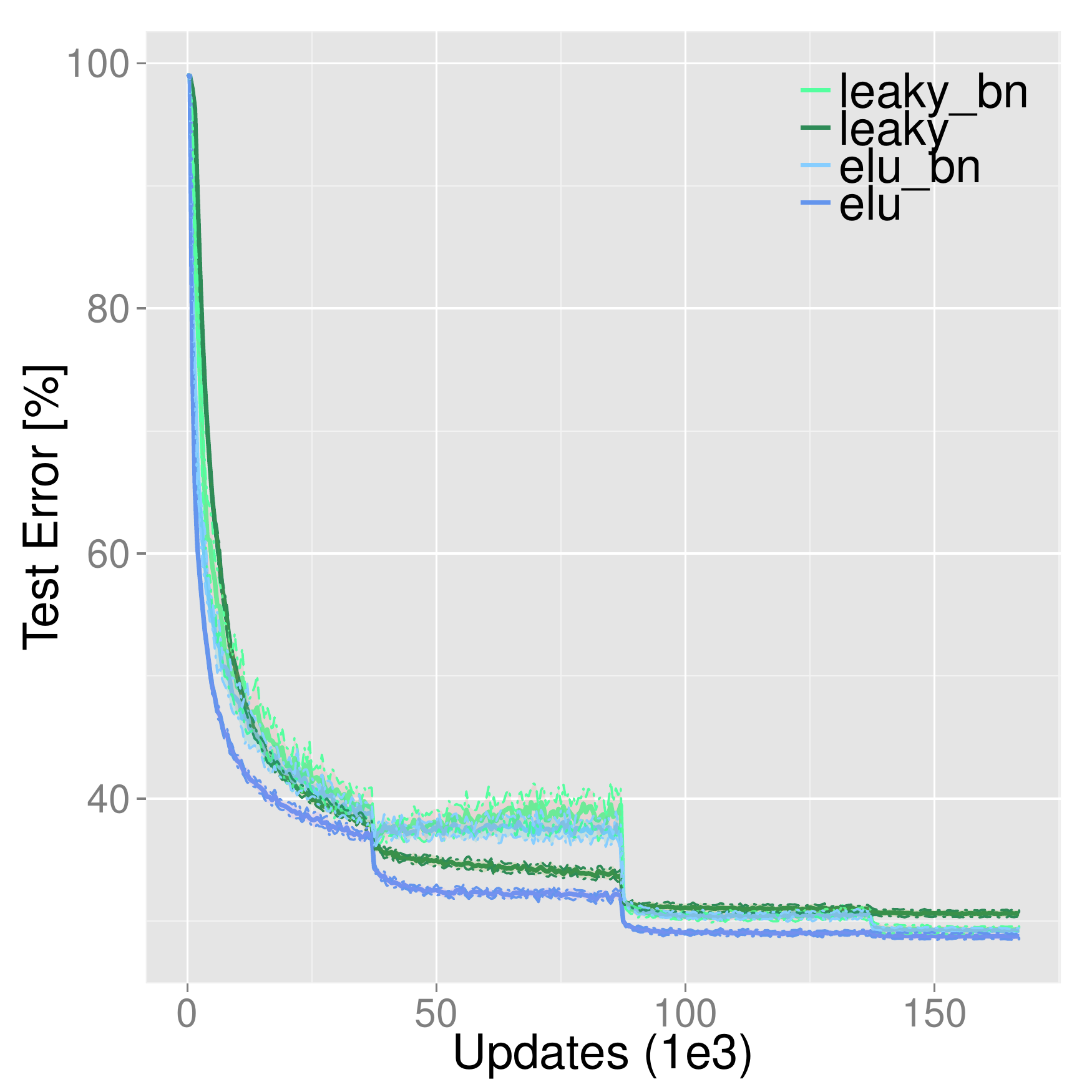}}\\[-2.0ex]
\subfigure[ELU - ReLU (end)]{
\includegraphics[angle=0,width= 0.32\textwidth]{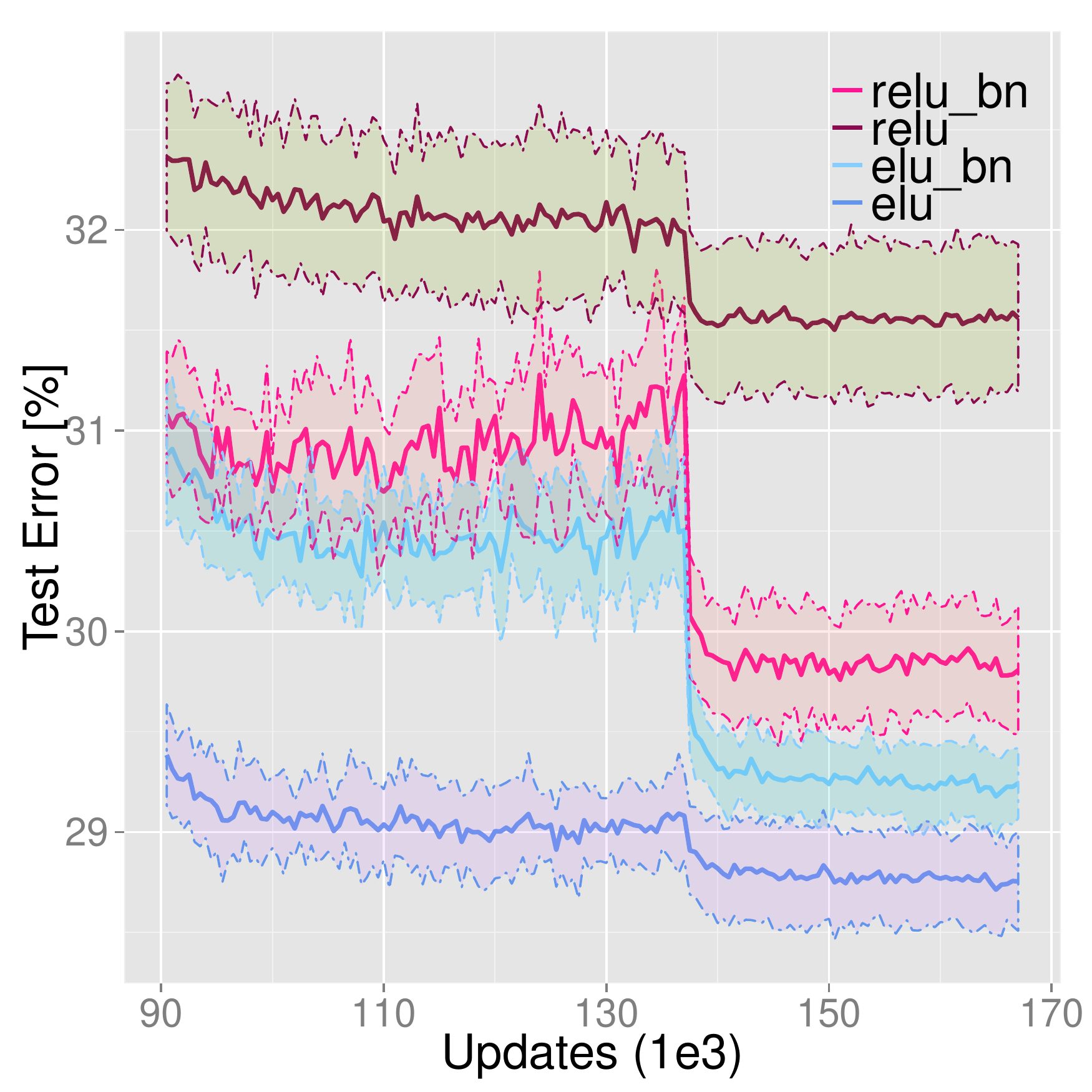}}
\subfigure[ELU - SReLU (end)]{
\includegraphics[angle=0,width= 0.32\textwidth]{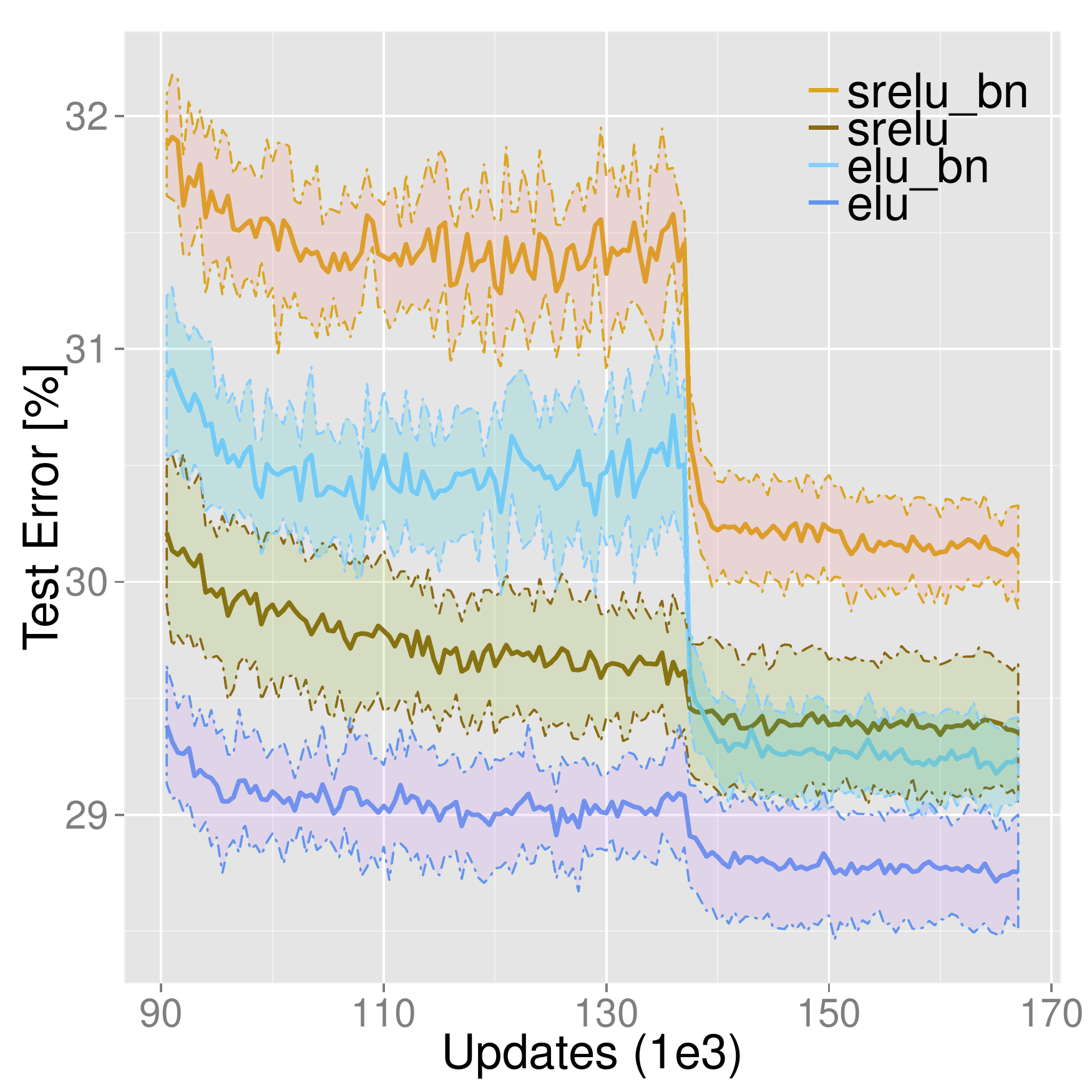}}
\subfigure[ELU - LReLU  (end)]{
\includegraphics[angle=0,width= 0.32\textwidth]{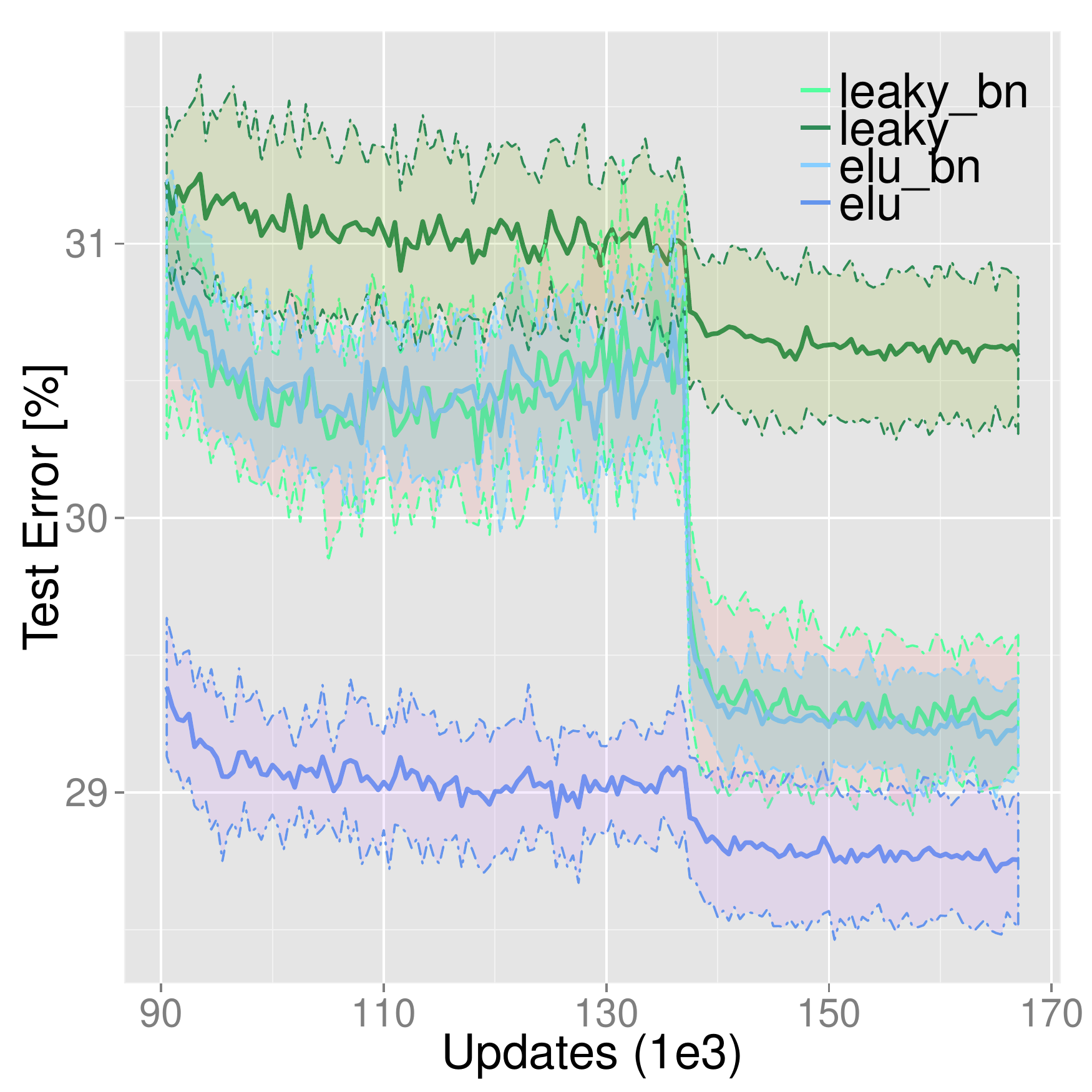}}
\caption{Pairwise comparisons of ELUs with ReLUs, SReLUs, and LReLUs with and without
batch normalization (BN) on CIFAR-100. Panels are described as in
Fig.~\ref{fig:resCIFAR100}.
ELU networks outperform ReLU networks with batch normalization.
\label{fig:resCIFAR100_BN}}
\end{center}
\vspace*{-5pt}
\end{figure*}

The CNN for these CIFAR-100
experiments consists of 11 convolutional layers arranged in stacks of
($[1\times192\times5], [1\times192\times1, 1\times240\times3],
[1\times240\times1, 1\times260\times2], [1\times260\times1,1\times280\times2],
[1\times280\times1,1\times300\times2], [1\times300\times1],
[1\times100\times1]$)
layers $\times$ units $\times$ receptive fields.
2$\times$2 max-pooling with a stride of 2 was applied after each
stack.  For network regularization we used the following drop-out rate
for the last layer of each stack ($0.0, 0.1, 0.2, 0.3, 0.4, 0.5,
0.0$). The $L2$-weight decay regularization term was set to $0.0005$.
The following learning rate schedule was applied ($0-35k [0.01],
35k-85k [0.005], 85k-135k [0.0005], 135k-165k [0.00005]$) (iterations [learning rate]).
For fair comparisons, we used this
learning rate schedule for all networks.
During previous experiments, this schedule was
optimized for ReLU networks, however as ELUs converge faster they
would benefit from an adjusted schedule.
The momentum term learning rate was fixed to 0.9.
The dataset was preprocessed as described in \citet{Goodfellow:13}
with global contrast normalization and ZCA whitening.
Additionally, the images were padded with
four zero pixels at all borders.
The model was trained on $32\times32$ random crops with random
horizontal flipping.
Besides that, we no further augmented the dataset during training.
Each network was run 10 times with different weight initialization.
Across networks with different activation functions the same run number
had the same initial weights.

Mean test error results of networks with different activation functions
are compared in Fig.~\ref{fig:resCIFAR100}, which also shows the standard
deviation.
ELUs yield on average a test error of 28.75($\pm$0.24)\%,
while SReLUs, ReLUs and LReLUs yield 29.35($\pm$0.29)\%,
31.56($\pm$0.37)\% and 30.59($\pm$0.29)\%, respectively.
ELUs achieve both lower training loss and lower test error than ReLUs,
LReLUs, and SReLUs.
Both the ELU training and test performance is significantly better than for other
activation functions (Wilcoxon signed-rank test with $p$-value$<$0.001).
Batch normalization improved ReLU and LReLU networks, but did not improve ELU and SReLU networks (see Fig.~\ref{fig:resCIFAR100_BN}).
ELU networks significantly outperform ReLU networks with batch
normalization (Wilcoxon signed-rank test with $p$-value$<$0.001).

\subsection{Classification Performance on CIFAR-100 and CIFAR-10}
\label{sec:CIFAR100_10}

The following experiments should highlight the generalization
capabilities of ELU networks.
The CNN architecture is more sophisticated than in the previous
subsection and consists of 18 convolutional layers arranged in stacks of
($[1\times384\times3], [1\times384\times1, 1\times384\times2,2\times640\times2],
[1\times640\times1,3\times768\times2], [1\times768\times1,2\times896\times2],
[1\times896\times3,2\times1024\times2], [1\times1024\times1,1\times1152\times2],
[1\times1152\times1],
[1\times100\times1]$).
Initial drop-out rate, Max-pooling after each stack, $L2$-weight
decay, momentum term, data preprocessing, padding, and cropping were as
in previous section.
The initial learning rate was set to 0.01 and decreased by a factor of
10 after 35k iterations. The mini-batch size was 100.
For the final 50k iterations fine-tuning we increased the drop-out
rate for {\em all} layers
in a stack to ($0.0, 0.1, 0.2, 0.3, 0.4, 0.5, 0.0$),
thereafter increased the drop-out rate by a factor of 1.5 for 40k additional iterations.
\begin{table}[th!]
\begin{center}
\caption{Comparison of ELU networks and other CNNs on
CIFAR-10 and CIFAR-100.
Reported is the test error in percent misclassification for ELU networks and recent convolutional
architectures like AlexNet, DSN, NiN, Maxout, All-CNN, Highway Network, and Fractional Max-Pooling.
Best results are in bold. ELU networks are second best for CIFAR-10
and best for CIFAR-100.}
\label{tab:tab_res1}%
\begin{tabular}{*{1}{>{\raggedright\arraybackslash}p{8em}}*{2}{>{\raggedleft\arraybackslash}p{11.5em}}*{1}{>{\raggedleft\arraybackslash}p{4.25em}}}
\toprule[1pt]
\addlinespace[2pt]
\bf{Network} &  \bf{CIFAR-10 (test error \%)} & \bf{CIFAR-100 (test error \%)} & \bf{augmented} \\
\toprule[1pt]
AlexNet  & 18.04 &  45.80 &\\
DSN   & 7.97   & 34.57   & $\surd$\\
NiN & 8.81 &  35.68 & $\surd$\\
Maxout & 9.38 & 38.57  & $\surd$\\
All-CNN & 7.25 & 33.71  & $\surd$\\
Highway Network & 7.60 & 32.24  & $\surd$\\
Fract. Max-Pooling & \bf{4.50} & 27.62  & $\surd$\\
ELU-Network & 6.55 & \bf{24.28} &\\
\bottomrule
\end{tabular}
\end{center}
\end{table}

ELU networks are compared to following
recent successful CNN
architectures: AlexNet \citep{Krizhevsky:12}, DSN \citep{Lee:15},
NiN \citep{Min:13}, Maxout \citep{Goodfellow:13},
All-CNN \citep{Springenberg:14}, Highway Network \citep{Srivastava:15}
and Fractional Max-Pooling \citep{Graham:14}.
The test error in percent misclassification are given in Tab.~\ref{tab:tab_res1}.
ELU-networks are the second best on CIFAR-10 with a test error of
6.55\% but still they are among the top 10 best results reported for CIFAR-10.
ELU networks performed best on CIFAR-100 with a test error of 24.28\%.
This is the best published result on CIFAR-100,
without even resorting to multi-view evaluation or model averaging.

\subsection{ImageNet Challenge Dataset}
\label{sec:ImageNet}

Finally, we evaluated ELU-networks on the 1000-class ImageNet dataset.
It contains about 1.3M training color images
as well as additional 50k images and 100k images for validation and
testing, respectively.
For this task, we designed a 15 layer CNN,
which was arranged in stacks of
($1\times96\times6, 3\times512\times3, 5\times768\times3,
3\times1024\times3, 2\times4096\times FC, 1\times1000\times FC$)
layers $\times$ units $\times$ receptive fields or fully-connected (FC).
2$\times$2 max-pooling with a stride of 2 was applied after
each stack and spatial pyramid pooling (SPP) with 3 levels before
the first FC layer \citep{He:15}.
For network regularization we set the $L2$-weight decay term to $0.0005$ and
used 50\% drop-out in the two penultimate FC layers.
Images were re-sized to 256$\times$256 pixels and per-pixel mean subtracted.
Trained was on $224\times224$ random crops with random
horizontal flipping. Besides that, we did not augment the dataset during training.
\begin{figure*}[!ht]
\begin{center}
\subfigure[Training loss]{
\includegraphics[angle=0,width= 0.32\textwidth]{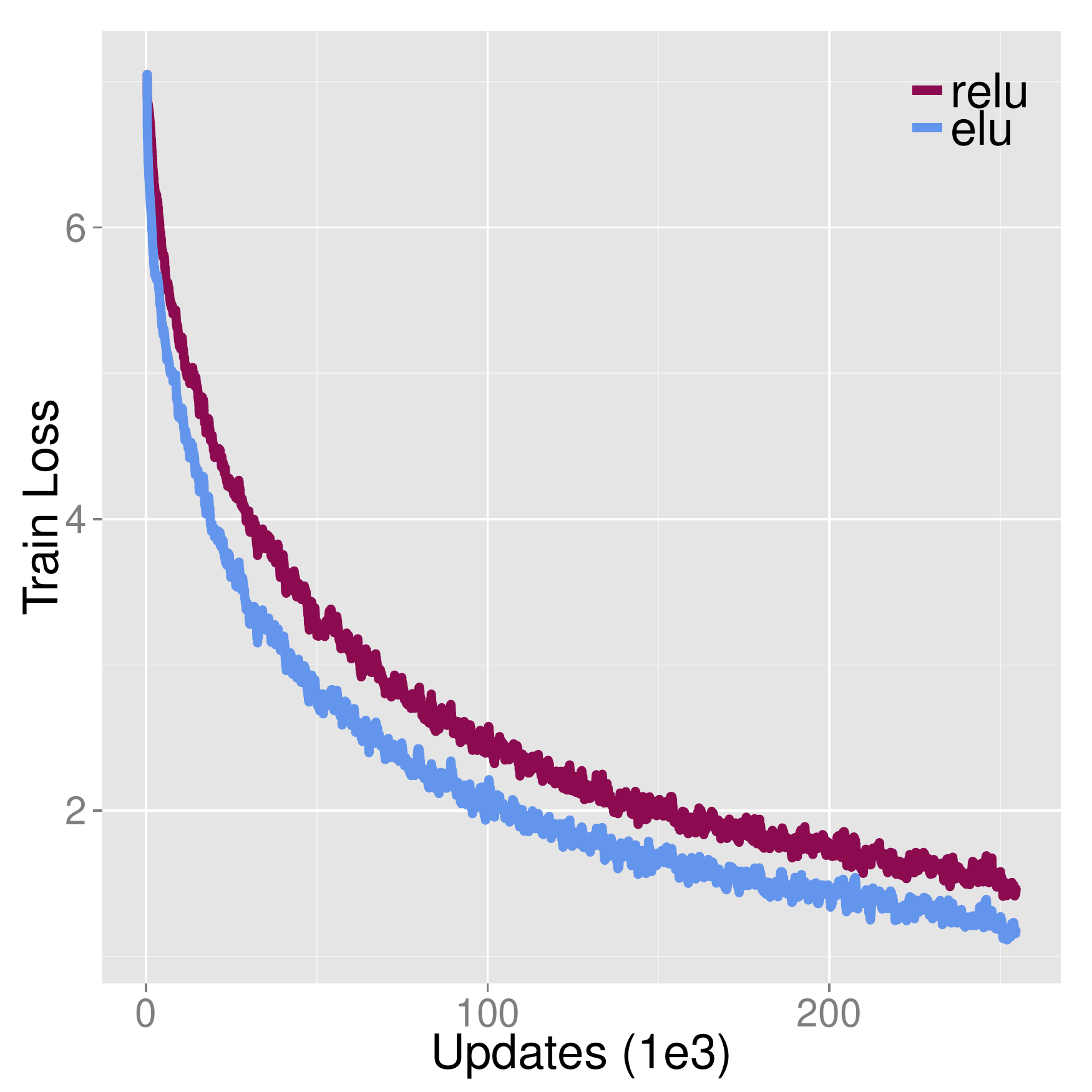}}
\subfigure[Top-5 test error]{
\includegraphics[angle=0,width= 0.32\textwidth]{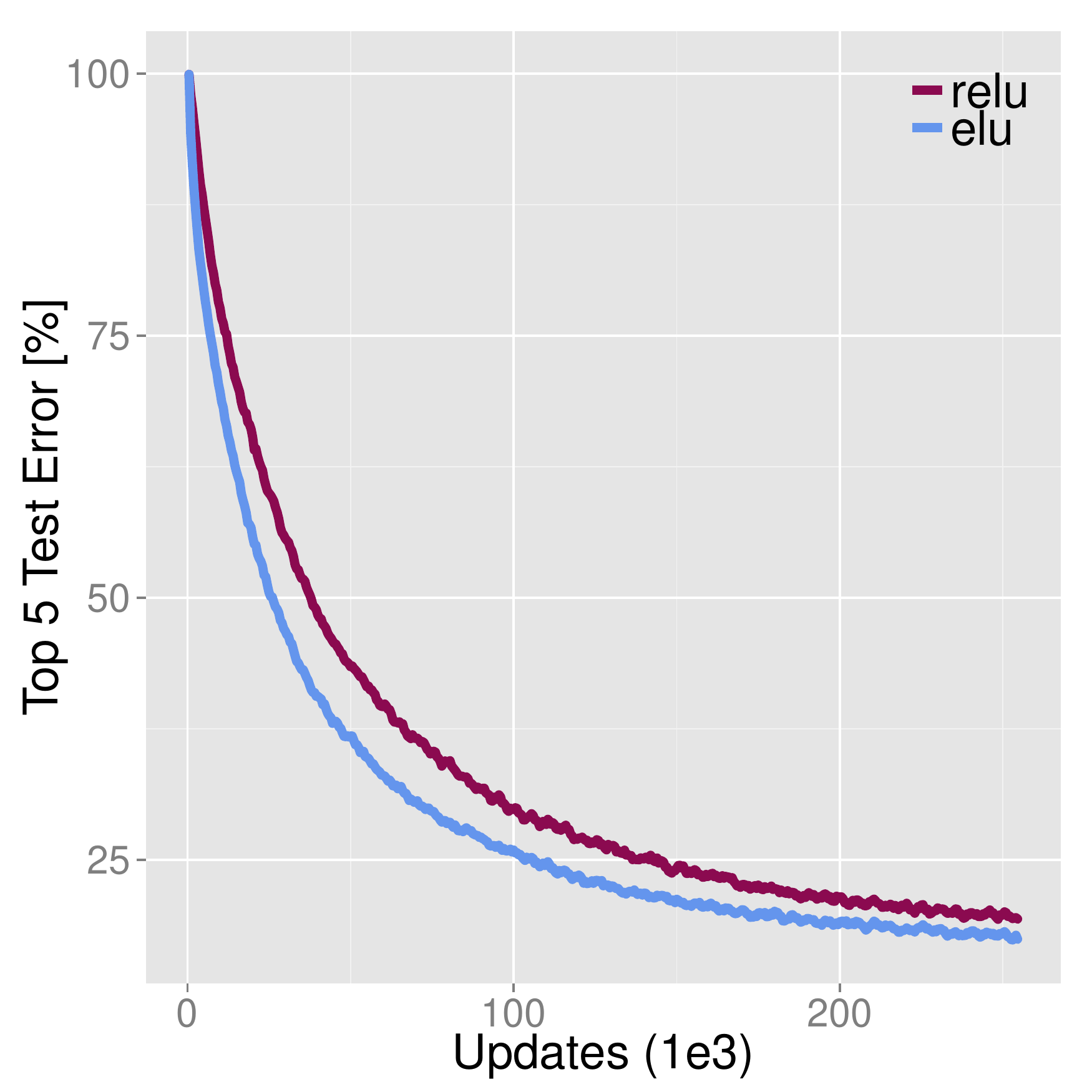}}
\subfigure[Top-1 test error]{
\includegraphics[angle=0,width= 0.32\textwidth]{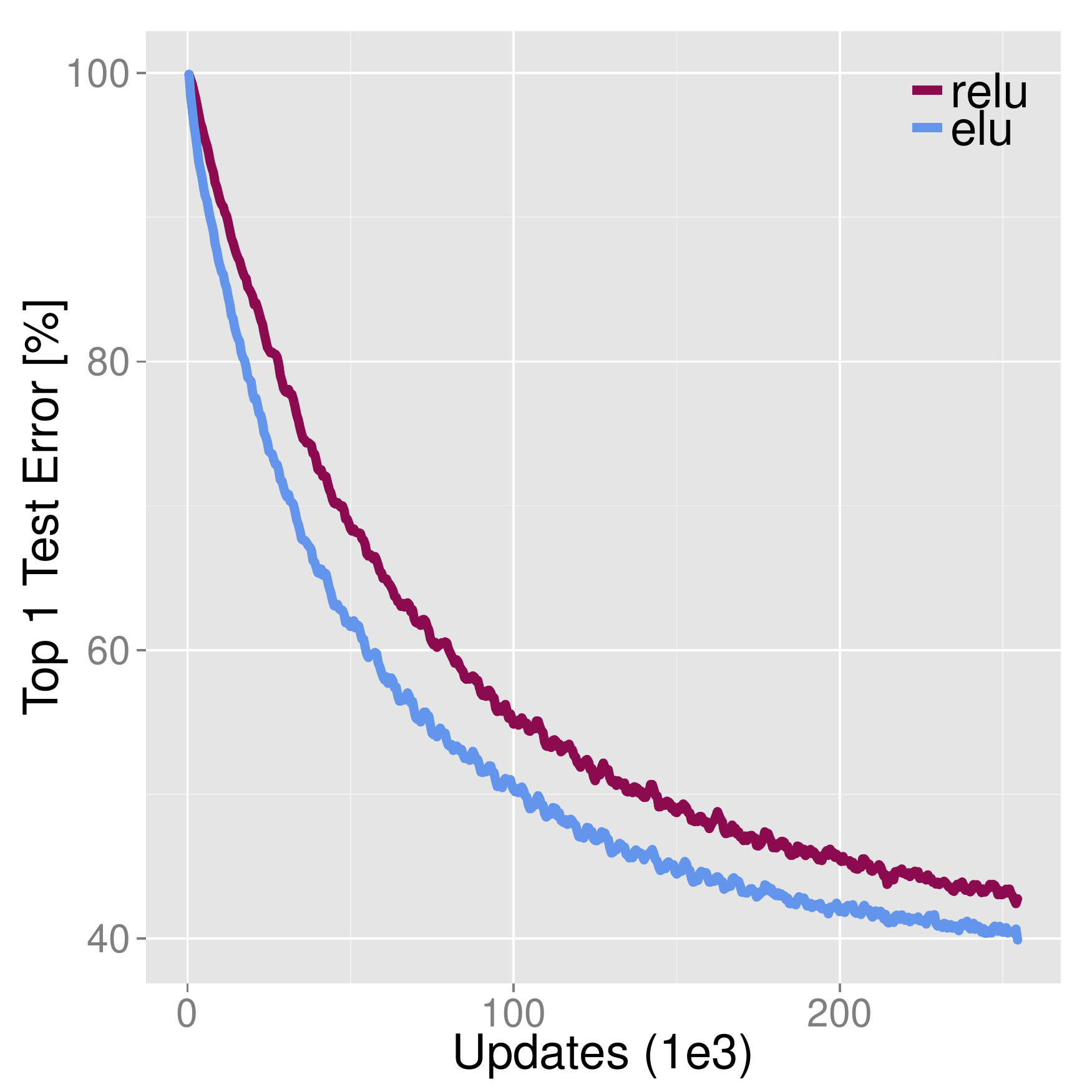}}
\caption{ ELU networks applied to ImageNet.
The $x$-axis gives the number of iterations and the $y$-axis the (a)
training loss, (b) top-5 error, and (c) the top-1 error
of 5,000 random validation samples, evaluated on the center crop.
Both activation functions ELU (blue) and ReLU (purple) lead for
convergence, but ELUs start reducing the error earlier and
reach the 20\% top-5 error after 160k iterations,
while ReLUs need 200k iterations to reach the same error rate. \label{fig:resImageNet}}
\end{center}
\vspace*{-5pt}
\end{figure*}

Fig.~\ref{fig:resImageNet} shows the learning behavior of
ELU vs.\ ReLU networks.
Panel (b) shows that ELUs start reducing the error earlier.
The ELU-network already reaches the 20\% top-5 error after 160k
iterations, while the ReLU network needs 200k
iterations to reach the same error rate.
The single-model performance was evaluated on the single center crop
with no further augmentation and yielded a top-5 validation error
below 10\%.

Currently ELU nets are 5\% slower on ImageNet than ReLU nets.
The difference is small because
activation functions generally have only minor influence on the overall
training time \citep{Jia:14}.
In terms of wall clock time, ELUs require 12.15h vs.\ ReLUs with
11.48h for 10k iterations.
We expect that ELU implementations can be improved, e.g.\ by faster
exponential functions \citep{Schraudolph:99}.

\section{Conclusion}
\label{sec:concl}
We have introduced the {\em exponential linear units} (ELUs) for faster and
more precise learning in deep neural networks.
ELUs have negative values, which allows the network to push
the mean activations closer to zero. Therefore ELUs decrease the gap
between the normal gradient and the unit natural gradient and, thereby
speed up learning.
We believe that this property is also the reason for the success
of activation functions like LReLUs and PReLUs and of batch
normalization. In contrast to LReLUs and PReLUs,
ELUs have a clear saturation plateau in its negative regime, allowing them
to learn a more robust and stable representation.
Experimental results show
that ELUs significantly outperform other activation functions
on different vision datasets. Further ELU networks perform
significantly better than ReLU networks trained with batch normalization.
ELU networks achieved one of the top 10 best reported results on CIFAR-10 and
set a new state of the art in
CIFAR-100 without the need for multi-view test evaluation or model averaging.
Furthermore, ELU networks produced competitive results on the ImageNet in much
fewer epochs than a corresponding ReLU network.
Given their outstanding performance, we expect
ELU networks to become a real time saver in convolutional networks, which are
notably time-intensive to train from scratch otherwise.

\paragraph*{Acknowledgment.} We thank the NVIDIA Corporation for supporting this research
with several Titan X GPUs and Roland Vollgraf and Martin Heusel for helpful discussions and
comments on this work.

\small{
\setlength{\bibsep}{1.5pt}
\bibliographystyle{iclr2016_conference}
\bibliography{activation}
}

\appendix

\section{Inverse of Block Matrices}
\label{sec:inverse}

\begin{lemma}
\label{th:lemma1}
The positive definite matrix $\BM$ is in block format with
matrix $\BA$, vector $\Bb$, and scalar $c$.
The inverse of $\BM$ is
\begin{align}
\BM^{-1} \ &= \
\begin{pmatrix}
\BA & \Bb \\
\Bb^T & c
\end{pmatrix}^{-1} \ = \
\begin{pmatrix}
\BK & \Bu \\
\Bu^T & s
\end{pmatrix} \ ,
\end{align}
where
\begin{align}
\BK \ &= \ \BA^{-1}  \ + \  \Bu \ s^{-1} \Bu^T \\
\Bu \ &= \ - \  s \ \BA^{-1} \ \Bb \\
s \ &= \ \left( c \ - \ \Bb^T\BA^{-1} \Bb \right)^{-1} \ .
\end{align}
\end{lemma}

\begin{proof}

For block matrices the inverse is
\begin{align}
\begin{pmatrix}
\BA & \BB \\
\BB^T & \BC
\end{pmatrix}^{-1} \ &= \
\begin{pmatrix}
\BK & \BU \\
\BU^T & \BS
\end{pmatrix} \ ,
\end{align}
where the matrices on the right hand side are:
\begin{align}
\BK \ &= \ \BA^{-1}  \ + \  \BA^{-1} \ \BB \ \left( \BC \ - \ \BB^T\BA^{-1} \BB
\right)^{-1} \BB^T \ \BA^{-1} \\
\BU \ &= \ - \  \BA^{-1} \ \BB \ \left( \BC \ - \ \BB^T\BA^{-1} \BB
\right)^{-1} \\
\BU^T \ &= \ - \ \left( \BC \ - \ \BB^T\BA^{-1} \BB
\right)^{-1}  \BB^T \ \BA^{-1} \\
\BS \ &= \ \left( \BC \ - \ \BB^T\BA^{-1} \BB
\right)^{-1} \ .
\end{align}
Further if follows that
\begin{align}
\BK   \ &= \ \BA^{-1} \ + \ \BU \ \BS^{-1} \BU^T \ .
\end{align}

We now use this formula for
$\BB=\Bb$ being a vector and $\BC=c$ a scalar.
We obtain
\begin{align}
\begin{pmatrix}
\BA & \Bb \\
\Bb^T & c
\end{pmatrix}^{-1} \ &= \
\begin{pmatrix}
\BK & \Bu \\
\Bu^T & s
\end{pmatrix} \ ,
\end{align}
where the right hand side matrices, vectors, and the scalar $s$ are:
\begin{align}
\BK \ &= \ \BA^{-1}  \ + \  \BA^{-1} \ \Bb \ \left( c \ - \ \Bb^T\BA^{-1} \Bb
\right)^{-1} \Bb^T \ \BA^{-1} \\
\Bu \ &= \ - \  \BA^{-1} \ \Bb \ \left( c \ - \ \Bb^T\BA^{-1} \Bb
\right)^{-1} \\
\Bu^T \ &= \ - \ \left( c \ - \ \Bb^T\BA^{-1} \Bb
\right)^{-1}  \Bb^T \ \BA^{-1} \\
s \ &= \ \left( c \ - \ \Bb^T\BA^{-1} \Bb
\right)^{-1} \ .
\end{align}
Again it follows that
\begin{align}
\BK   \ &= \ \BA^{-1} \ + \ \Bu \ s^{-1} \Bu^T \ .
\end{align}

A reformulation using $\Bu$ gives
\begin{align}
\BK \ &= \ \BA^{-1}  \ + \  \Bu \ s^{-1} \Bu^T \\
\Bu \ &= \ - \  s \ \BA^{-1} \ \Bb \\
\Bu^T \ &= \ - \ s \ \Bb^T \ \BA^{-1} \\
s \ &= \ \left( c \ - \ \Bb^T\BA^{-1} \Bb \right)^{-1} \ .
\end{align}
\end{proof}

\section{Quadratic Form of Mean and Inverse Second Moment}
\label{sec:second}

\begin{lemma}
\label{th:lemma2}
For a random variable $\Ba$ holds
\begin{align}
&\EXP^T(\Ba) \ \EXP^{-1}(\Ba \ \Ba^T) \ \EXP(\Ba) \ \leq \ 1
\end{align}
and
\begin{align}
&\left( 1 \ - \ \EXP^T(\Ba) \ \EXP^{-1}(\Ba \ \Ba^T) \ \EXP(\Ba)
\right)^{-1} \ = \ 1 \ + \ \EXP^T(\Ba) \ \VAR^{-1}(\Ba) \ \EXP(\Ba) \ .
\end{align}
Furthermore holds
\begin{align}
&\left( 1 \ - \ \EXP^T(\Ba) \ \EXP^{-1}(\Ba \ \Ba^T) \ \EXP(\Ba)
\right)^{-1} \ \left(1 \ - \ \EXP_p^T(\Ba) \ \EXP^{-1}(\Ba \ \Ba^T) \ \EXP(\Ba)\right)
\\ \nonumber
&= \ 1 \ + \  \left( \EXP(\Ba) \ - \ \EXP_p(\Ba) \right)^T
\VAR^{-1}(\Ba) \ \EXP(\Ba) \ .
\end{align}
\end{lemma}

\begin{proof}
The Sherman-Morrison Theorem states
\begin{align}
&\left(\BA \ + \ \Bb \ \Bc^T \right)^{-1} \ = \ \BA^{-1} \ - \ \frac{\BA^{-1}
    \ \Bb \ \Bc^T \ \BA^{-1}}{1 \ + \ \Bc^T \BA^{-1} \Bb} \ .
\end{align}
Therefore we have
\begin{align}
\label{eq:sherman}
&\Bc^T \left(\BA \ + \ \Bb \ \Bb^T\right)^{-1} \Bb \ = \ \Bc^T \BA^{-1} \Bb \ - \ \frac{\Bc^T \BA^{-1}
    \ \Bb \ \Bb^T \ \BA^{-1} \Bb}{1 \ + \ \Bb^T \BA^{-1} \Bb} \\ \nonumber
&= \ \frac{\Bc^T \BA^{-1} \Bb \ \left( 1 \ + \ \Bb^T \BA^{-1} \Bb
  \right) \ - \ \left(\Bc^T \BA^{-1}
    \ \Bb\right)\left(\Bb^T \BA^{-1}
    \ \Bb\right)  }{1 \ + \ \Bb^T \BA^{-1} \Bb} \\ \nonumber
&= \ \frac{\Bc^T \BA^{-1} \Bb}{1 \ + \ \Bb^T \BA^{-1} \Bb} \ .
\end{align}

Using the identity
\begin{align}
&\EXP(\Ba \ \Ba^T) \ = \ \VAR(\Ba) \ + \ \EXP(\Ba) \ \EXP^T(\Ba)
\end{align}
for the second moment and Eq.~\eqref{eq:sherman}, we get
\begin{align}
&\EXP^T(\Ba) \ \EXP^{-1}(\Ba \ \Ba^T) \ \EXP(\Ba)
\ = \ \EXP^T(\Ba) \ \left( \VAR(\Ba) \ + \ \EXP(\Ba) \ \EXP^T(\Ba)
\right)^{-1} \ \EXP(\Ba) \\ \nonumber
&= \ \frac{\EXP^T(\Ba) \ \VAR^{-1}(\Ba) \ \EXP(\Ba)}{1 \ + \
  \EXP^T(\Ba) \ \VAR^{-1}(\Ba) \ \EXP(\Ba)}
\ \leq \ 1 \ .
\end{align}
The last inequality follows from the fact that $\VAR(\Ba)$ is positive definite.
From last equation, we obtain further
\begin{align}
&\left( 1 \ - \ \EXP^T(\Ba) \EXP^{-1}(\Ba \ \Ba^T) \ \EXP(\Ba)
\right)^{-1} \ = \ 1 \ + \ \EXP^T(\Ba) \VAR^{-1}(\Ba) \ \EXP(\Ba) \ .
\end{align}

For the mixed quadratic form we get from Eq.~\eqref{eq:sherman}
\begin{align}
&\EXP_p^T(\Ba) \ \EXP^{-1}(\Ba \ \Ba^T) \ \EXP(\Ba)
\ = \ \EXP_p^T(\Ba) \ \left( \VAR(\Ba) \ + \ \EXP(\Ba) \ \EXP^T(\Ba)
\right)^{-1} \ \EXP(\Ba) \\ \nonumber
&= \ \frac{\EXP_p^T(\Ba) \ \VAR^{-1}(\Ba) \ \EXP(\Ba)}{1 \ + \
  \EXP^T(\Ba) \ \VAR^{-1}(\Ba) \ \EXP(\Ba)}
\ .
\end{align}
From this equation follows
\begin{align}
&1 \ - \ \EXP_p^T(\Ba) \ \EXP^{-1}(\Ba \ \Ba^T) \ \EXP(\Ba)
\ = \ 1 \ - \ \frac{\EXP_p^T(\Ba) \ \VAR^{-1}(\Ba) \ \EXP(\Ba)}{1 \ + \
  \EXP^T(\Ba) \ \VAR^{-1}(\Ba) \ \EXP(\Ba)}
 \\ \nonumber
&= \ \frac{1 \ + \  \EXP^T(\Ba) \ \VAR^{-1}(\Ba) \ \EXP(\Ba) \ - \
  \EXP_p^T(\Ba) \ \VAR^{-1}(\Ba) \ \EXP(\Ba)}{1 \ + \ \EXP^T(\Ba) \ \VAR^{-1}(\Ba) \ \EXP(\Ba)}
\ = \ \frac{1 \ + \  \left( \EXP(\Ba) \ - \ \EXP_p(\Ba) \right)^T
  \VAR^{-1}(\Ba) \ \EXP(\Ba) }{1 \ + \ \EXP^T(\Ba) \ \VAR^{-1}(\Ba) \ \EXP(\Ba)}
\ .
\end{align}

Therefore we get
\begin{align}
&\left( 1 \ - \ \EXP^T(\Ba) \ \EXP^{-1}(\Ba \ \Ba^T) \ \EXP(\Ba)
\right)^{-1} \ \left(1 \ - \ \EXP_p^T(\Ba) \ \EXP^{-1}(\Ba \ \Ba^T) \ \EXP(\Ba)\right)
\\ \nonumber
&= \ 1 \ + \  \left( \EXP(\Ba) \ - \ \EXP_p(\Ba) \right)^T
\VAR^{-1}(\Ba) \ \EXP(\Ba) \ .
\end{align}

\end{proof}

\section{Variance of Mean Activations in ELU and ReLU Networks}
\label{sec:varianceMean}
To compare the variance of median activation in ReLU and ELU networks,
we trained a neural network with 5 hidden layers of 256 hidden units for 200
epochs using a learning rate of 0.01, once using ReLU and once using ELU
activation functions on the MNIST dataset.
After each epoch, we calculated the median activation
of each hidden unit on the whole training set.
We then calculated the variance of these changes, which is depicted in
Figure~\ref{fig:varianceofchanges}\, . The median varies much more in ReLU networks.
This indicates that ReLU networks continuously try to correct
the bias shift introduced by previous weight updates
while this effect is much less prominent in ELU networks.
\begin{figure*}[!ht]
\begin{center}
\includegraphics[width= 0.7\textwidth]{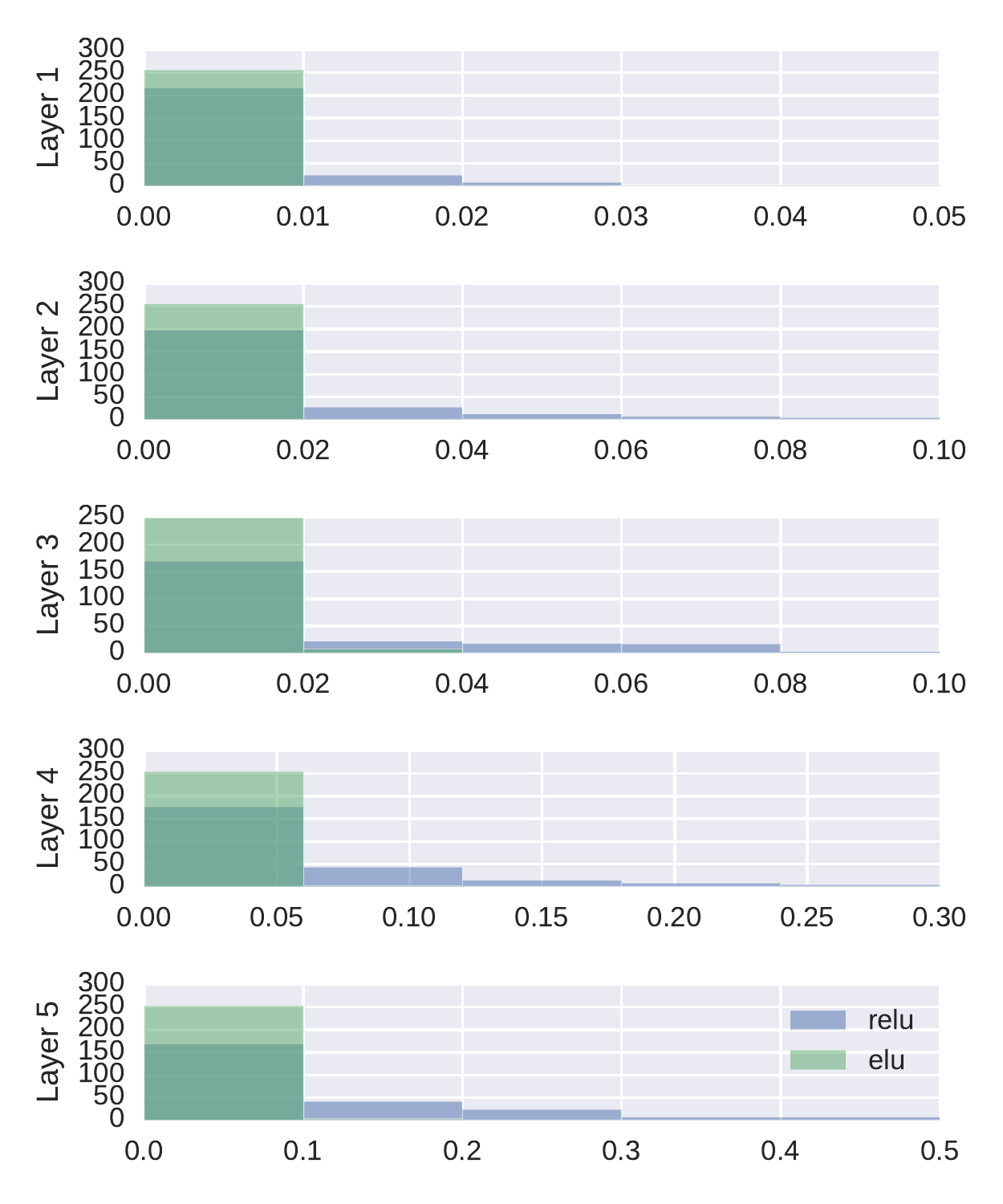}
\caption{Distribution of variances of the median hidden unit activation
after each epoch of MNIST training. Each row represents the units in
a different layer of the network.
\label{fig:varianceofchanges}}
\end{center}
\end{figure*}

\end{document}